\documentclass[sigconf]{acmart}

\AtBeginDocument{%
  \providecommand\BibTeX{{%
    \normalfont B\kern-0.5em{\scshape i\kern-0.25em b}\kern-0.8em\TeX}}}

\setcopyright{acmcopyright}
\copyrightyear{2023}
\acmYear{2023}
\acmDOI{10.1145/3580305.3599309}
\acmISBN{979-8-4007-0103-0/23/08}
\usepackage{algorithm}
\usepackage{algorithmic}
\acmConference[KDD '23]{Proceedings of The International ACM SIGIR Conference 2022}{August 6-10, 2023}{Long Beach, CA}
\acmBooktitle{Proceedings of The International ACM SIGKDD Conference 2023 (KDD '23), August 6-10, 2023, Long Beach, CA}

\newcommand{\ie}{\emph{i.e., }}
\newcommand{\eg}{\emph{e.g., }}

\newcommand{\wrt}{\emph{w.r.t. }}

\newcommand{\za}[1]{{\color{black}{#1}}}
\newcommand{\ff}[1]{{\color{black}{#1}}}

\usepackage{subfigure}
\usepackage{enumitem}
\usepackage{stfloats}
\begin{document}

\title{Discovering Dynamic Causal Space for DAG Structure Learning  }

\author{Fangfu Liu}
\orcid{https://orcid.org/0009-0007-9422-218X}
\affiliation{%
  \institution{Tsinghua University}
  \country{}
}
\email{liuff19@mails.tsinghua.edu.cn}

\author{Wenchang Ma}
\orcid{https://orcid.org/0009-0006-8264-4775}
\affiliation{%
  \institution{National University of Singapore}
  \country{}
  }
\email{e0724290@u.nus.edu}

\author{An Zhang}
\authornote{An Zhang is the corresponding author.}
\orcid{https://orcid.org/0000-0003-1367-711X}
\affiliation{%
 \institution{National University of Singapore}
 \country{}
 }
\email{anzhang@u.nus.edu}

\author{Xiang Wang}
\authornote{Xiang Wang is also affiliated with the Institute of Artificial Intelligence, Institute of Dataspace, and Hefei Comprehensive National Science Center.}
\orcid{https://orcid.org/0000-0002-6148-6329}
\affiliation{%
  \institution{University of Science and Technology of China}
  \country{}
  }
\email{xiangwang1223@gmail.com}

\author{Yueqi Duan}
\orcid{https://orcid.org/0000-0002-1190-6663}
\affiliation{%
  \institution{Tsinghua University}
  \country{}
  }
\email{duanyueqi@tsinghua.edu.cn}

\author{Tat-Seng Chua}
\orcid{https://orcid.org/0000-0001-6097-7807}
\affiliation{%
  \institution{National University of Singapore}
  \country{}
}
\email{dcscts@nus.edu.sg}
\renewcommand{\shortauthors}{Fangfu Liu et al.}
\begin{abstract}
Discovering causal structure from purely observational data (\ie causal discovery), aiming to identify causal relationships among variables, is a fundamental task in machine learning.The recent invention of differentiable score-based DAG learners is a crucial enabler, which reframes the combinatorial optimization problem into a differentiable optimization with a DAG constraint over directed graph space. Despite their great success, these cutting-edge DAG learners incorporate DAG-ness independent score functions to evaluate the directed graph candidates, lacking in considering graph structure. As a result, measuring the data fitness alone regardless of DAG-ness inevitably leads to discovering suboptimal DAGs and model vulnerabilities. 

Towards this end, we propose a dynamic \underline{ca}usal \underline{sp}ace for DAG structure l\underline{e}a\underline{r}ning, coined \textbf{CASPER}, that integrates the graph structure into the score function as a new measure in the causal space to faithfully reflect the causal distance between estimated and ground-truth DAG. CASPER revises the learning process as well as enhances the DAG structure learning via adaptive attention to DAG-ness. Grounded by empirical visualization, CASPER, as a space, satisfies a series of desired properties, such as structure awareness and noise-robustness. Extensive experiments on both synthetic and real-world datasets clearly validate the superiority of our CASPER over the state-of-the-art causal discovery methods in terms of accuracy and robustness. 
\end{abstract}

\ccsdesc[500]{Mathematics of computing~Causal networks}
\ccsdesc[500]{Computing methodologies~Causal reasoning and diagnostics}

\keywords{Differentiable Causal Discovery, Score-based Structure Learning, Score Function, DAG-ness Aware Scoring}



\maketitle

\section{Introduction}
Learning directed acyclic graph (DAG) structure from observational data (\ie causal discovery) is a fundamental problem~\cite{CD-survey-2022} in machine learning for a broad range of applications, including genetics~\cite{prob-graph-model-2009}, biology~\cite{corr2causal-2007}, economics~\cite{pearl-models-2000, pearl2010causal} and social science~\cite{social-science-2019}. 
The purpose of DAG structure learning is to discover causal relationships among a set of variables that are encoded in a DAG~\cite{spirtes2000causation}. 
Conventional score-based methods assess the directed graph \za{candidates} by utilizing a pre-defined score function over the DAG space~\cite{silander2012simple}. 
However, the intractable combinatorial nature of acyclic space frames DAG learning as a combinatorial optimization problem \wrt discrete edges, which has been proven to be NP-hard~\cite{NP-hard}. 
A recent breakthrough, NOTEARS~\cite{notears2018}, successfully transforms the discrete DAG constraint into a continuous equality constraint, resulting in a differentiable optimization \za{framework with an acyclic regularization term}.
Following \za{differentiable causal discovery methods}~\cite{DAG-GNN-2019, gran-dag, RL-CD, notears-mlp, bhattacharya2021differentiable, NoCurl}, inspired by NOTEARS~\cite{notears2018}, optimize the score function by leveraging \za{various} highly parameterized deep networks \za{via} gradient descent.
\za{Though effective, these cutting-edge methods inevitably simplify the searching space from DAG space to directed graph space, further increasing the risk of discovering suboptimal DAGs. }

\begin{figure*}[!t]
    \centering
    \includegraphics[width=0.9\linewidth]{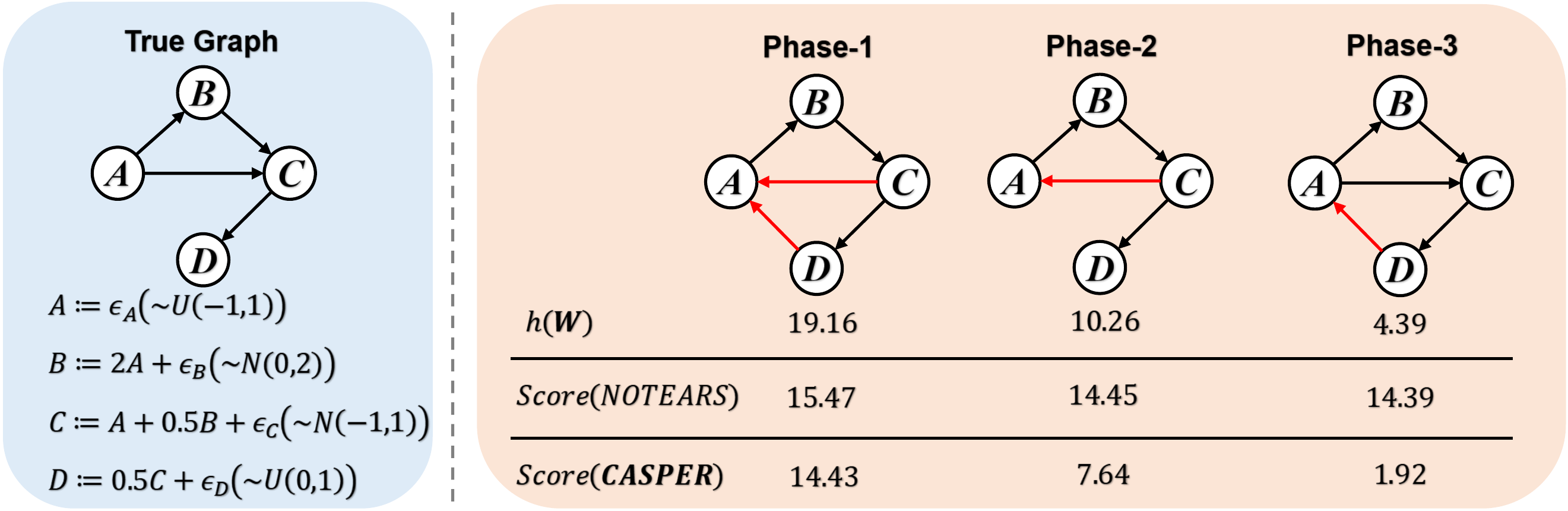}
    \caption{
    An illustrative example of the DAG learning progress is that NOTEARS may yield the same scores for different DAG-ness graphs across different optimization phases, each parameterized by $h(\mathbf{W})$.
    The values of $h(\mathbf{W})$ defined in Equation \eqref{eq: differentiable-score} quantify the extent of violations of acyclicity as the weighted matrix $\mathbf{W}$ deviates further from DAGs. 
    Consequently, NOTEARS-based methods fail to quantify the intrinsic causal distance by conventional score function. 
    In contrast, our method CASPER can dynamically perceive the DAG structure and score the models based on the underlying causal relationships, further guiding the DAG structure learning.
    }
    \label{fig:teaser}
\end{figure*}

\za{This motivates us to rethink the framework of score-based differentiable causal discovery, aiming to infer the causal structure model that encodes both graph structure (\ie DAG) and data mapping (\ie structural equations).}
\za{Three essential components comprise the most recent differentiable score-based DAG learner: score function, DAG constraint, and deep networks for gradient-based optimization.}
\za{DAG constraint, a hard penalty, quantifies the DAG-ness of graph and its coefficient has to go to infinity to impose acyclicity, whereas, in most training processes, the DAG learner searches in directed graph space.}
\za{While prevalent score functions, such as least square loss~\cite{notears2018, notears-mlp} and maximum log-likelihood estimator~\cite{gran-dag}, only evalute the goodness-of-fit, which describes how well the data fit the estimated structure equations. }
\za{In other words, the majority of existing score functions neglect the graph structure and merely evaluate the data fitness using static metrics \cite{static-space} for all the directed graphs, regardless of the DAG-ness degree}. 
\za{To ensure today's score functions appropriately evaluate the candidate directed graphs, these static metrics are implicitly defined in a fixed scoring space holding an inherent assumption, \ie that the estimated directed graph is an acyclic graph throughout the training process.}
\za{However, the differentiable optimization framework makes it simple to violate this underlying assumption.} 
\za{We conclude that measuring the static fitness of the structural equations regardless of DAG-ness fails to quantify the intrinsic distance between the estimated directed graph and the optimal DAG~\cite{notears2018, kaiser2021unsuitability}. This contributes to learning a suboptimal DAG, also leading to the lack of noise-robustness in DAG learning~\cite{kaiser2021unsuitability, DAG-GNN-2019}.}

\za{In this paper, we conjecture that an ideal score function should not only successfully measure the data fitness but also take the graph structure into consideration.}
We further substantiate our claim with an illustrative example as shown in Figure~\ref{fig:teaser}. 
\za{In this example, the estimated directed graph is forced to be more acyclic (\ie the value of $h(\mathbf{W})$ drops) as the intensity of the DAG constraint grows in the training process. 
The scores of NOTEARS in phases 2 and 3, focusing solely on data fitness, are comparatively indiscriminate.
Unfortunately, facing an exponentially increasing of  DAG constraint coefficient, NOTEARS is likely to encounter ill-conditioning issues and slip into the local minimum in phase 2. 
In contrast, benefiting from taking graph structure into account, it is considerably easier for CASPER to pass phase 2, avoiding falling into a local optimal graph.
Hence, such DAG-ness independent score fucntions hardly reveal the distance between the estimated and ground-truth DAG, being at odds with the true learning process under causal structure space.
}


\za{Motivated by the limitation of DAG-ness independent score functions, we attempt to discover a simple form of score function encoding the graph structural information. We consider DAG structure learning in a dynamic space from a distributional view as opposed to using the static metrics of existing score functions.  Specifically, the desirable score function, as the measure in this new space, has causal semantics, indicating that by incorporating information from structural equation models, it can accurately reflect DAG-ness of candidate graphs.  The DAG-ness-aware property of the new score function enables us to alleviate the local minimum issue and helps reconstruct a more precise DAG. As a result, the dynamical space can measure intrinsic causal relationships and data fitness from the perspective of distribution, which can further enhance the robustness of the DAG learner.}


Guided by this idea, we propose a dynamic \underline{ca}usal \underline{sp}ace for DAG structure l\underline{e}a\underline{r}ning, coined \textbf{CASPER}, which satisfies a series of good properties, including complete probability metric and noise-robustness. 
In this paper, firstly, we develop a descriptor that encapsulates the graph structural equation. 
By inserting the descriptor of the causal space, CASPER may dynamically adjust the complexity of the measure in accordance with DAG-ness in the optimization process. 
Secondly, we use the measure (also known as causal distance in our paper) defined in causal space as the primary component of the score function. As a result, we may adaptively perceive the DAG-ness of candidate graphs in the training process. Thirdly, we define the Boral probability measure in our causal space, which may accurately reflect the sampling distribution while remaining faithful to the DAG. Our causal space is therefore robust to the distortion caused by noise in observational data.


In summary, our contributions are highlighted as:
\begin{itemize}[leftmargin=*]
    \item To the best of our knowledge, we are among the first class to impose the DAG-ness-aware information into the framework of differentiable DAG structure learning.
    \item
    We propose a novel optimization scheme for DAG structure learning called CASPER. CASPER encodes the graph structure of the structural equation model using a dynamic causal space, allowing us to enhance the score function with adaptive attention to the causal structure.
    \item Extensive experiments both on synthetic and real-world datasets demonstrate our proposed method can significantly improve the performance of existing causal discovery models.
\end{itemize}


\section{Algorithm}
Prevailing algorithms for DAG structure learning can be broadly categorized into two research lines: constraint-based methods~\cite{heterogeneous-2020} and score-based methods~\cite{glymour2019review, CD-survey-2022}. 
Constraint-based approaches always test for conditional independencies according to the empirical joint distribution under certain assumptions~\cite{spirtes2000causation, causal-from-soft-intervention}, in order to construct a graph that reflects these conditional independencies. 
On the other hand, the score-based approaches evaluate the validity of a candidate graph $\mathcal{G}$ under some predefined score function~\cite{nofears}. 
In this paper, our focus is primarily on differentiable score-based algorithms. 
Before introducing our CASPER, we provide a brief overview of the fundamental concepts in DAG structure learning.

\subsection{Problem Definition}
DAG structure learning (\ie causal discovery) aims to infer the Structural Equation Model (SEM)~\cite{pearl-models-2000, causal-inference-statistics} from the observational data, which models the data generating procedure. 
Formally, the basic DAG structure learning problem is formulated as follows: Let $\mathbf{X} = [\mathbf{x_1}|...|\mathbf{x_d}]\in \mathbb{R}^{n\times d}$ be a data matrix consisting of $n$ i.i.d. observational data of $d$ variables. 
And $\mathcal{H}$ denotes the space of DAGs $\mathcal{G} = (\mathcal{V}, \mathcal{E})$ on $d$ nodes, where $\mathcal{V}$ represents the set of node variables, denoted as $X = {(X_1,..., X_d)}$ and $\mathcal{E}$ is the set of cause-effect edges between variables. Given $\mathbf{X}$, we try to learn a directed acyclic graph (DAG) $\mathcal{G}\in \mathcal{H}$ for the joint distribution $P(X)$ \citep{spirtes2000causation, prob-graph-model-2009}. To model $X$, we consider a generalized structural equation model (SEM) as follows:
\begin{equation}
    X_j := f_j(X_{pa(X_j)}) + N_j , \quad j\in \{1,...d\} ,
    \label{eq: SEM}
\end{equation}

where $X_j$ is the $j$-th node variable, $pa(X_j)$ denote the parents of $X_j$, $f_j$ is the causal structure funntion, and $N_j$ refers to the additive noise with variance $\sigma_j^2$. Without loss of generality, the observed data $\mathbf{X}$ can be regarded as the samples from the joint distribution $P(X)$, our goal is to use the samples to reconstruct the underlying causal structure represented by DAG $\mathcal{G}$. 

\subsection{Preliminary}
\label{subsec: preliminary}
\textbf{Structure Identifiability.} Unraveling the identifiability of causal direction is a crucial issue in the process of DAG structure learning. In general, it is impossible to reconstruct $\mathcal{G}$ given only observational samples from $P(X)$ if we do not impose any assumptions on SEMs (\ie Equation~\eqref{eq: SEM}). Considering a set of assumptions $\mathcal{A}$ over a causal graphical model $\mathcal{M}_\mathcal{A} = (P_X, \mathcal{G})$, the graph $\mathcal{G}$ is identifiable from $P(X)$ if and only if there is no other $\tilde{\mathcal{M}}_\mathcal{A} = (\tilde{P}_X, \tilde{\mathcal{G}})$ satisfying the same $\mathcal{A}$ such that $\tilde{\mathcal{G}} \neq \mathcal{G}$ and $\tilde{P}(X) = P(X)$. To satisfy the identifiability of the graph, researchers~\cite{loh2014high-dimension-CD, pena2018identifiability, peters2014causal-with-adm} always assume that the conditional densities belong to a specific parametric family (\eg additive noise models).
\newline

\noindent \textbf{Score-based Structure Learning.} The goal of conventional score-based structure learning is stated as the following combinatorial optimization problem~\cite{element-of-causal-inference}:
\begin{equation}
\begin{aligned}
 \min _{\mathcal{G}} F( \mathbf{X}; \mathcal{G}) & =\mathcal{L}_{\text{rec}}(\mathbf{X}; \mathcal{G})+\lambda \mathcal{R}_{\text {sparse }}(\mathcal{G}), \\
&\text { s.t. } \mathcal{G} \in \text { DAGs, }
\end{aligned}
\label{eq: common-score-def}
\end{equation}
where $F$ is a score function, $\mathbf{X}$ symbolizes the observational data and $\mathcal{G}$ refers to a directed graph. 
We notice that the score function $F( \mathbf{X}; \mathcal{G})$ consists of two terms: (1) the measure of graph reconstruction, referred to as the proximity of the optimized graph to the true DAG; (2) the sparsity regularization term, represented as $\mathcal{R}_{\text{sparse}}(\mathcal{G})$, mandates that the count of edges in the graph be subject to a penalty, typically achieved through $l_1$ regularization in practice. 
The hyperparameter $\lambda$ plays a pivotal role in modulating the significance of the regularization process. 
Common score functions include the MDL~\cite{MDL}, BIC~\cite{BIC-score} and BGe~\cite{BGe}. 
However, the Equation~\eqref{eq: common-score-def} is NP-hard to solve due to the nonconvex and combinatorial nature of the optimization problem~\cite{chickering2004large-nphard, chickering1996learning-nphard}. To address the combinatorial problem, Zheng et al.~\cite{notears2018} convert it into a continuous program as:
\begin{equation}
    \min _{\mathcal{G}} F( \mathbf{X}; \mathcal{G} ) \quad \text { s.t. } \quad h(\mathcal{G})=0 \text{,}
    \label{eq: differentiable-score}
\end{equation}
where $h(\mathcal{G})=0$ is a differentiable function over real matrices, whose level set at zero exactly means the acyclicity of a graph. Note that, there are various alternatives of $h(\mathcal{G})$~\cite{notears2018, DAG-GNN-2019, nofears, kyono2020castle} in literature. Therefore, we went from the combinatorial optimization problem to a continuous constrained optimization problem. Fortunately, numerous solutions (\eg augmented Lagrangian method~\cite{bertsekas1997nonlinear-lagrange, nemirovsky1999optimization-lagrange}) can be applied to solve the Equation~\eqref{eq: differentiable-score}. 
As a result, the optimization problem in Equation~\eqref{eq: differentiable-score} can be further reformulated as:
\begin{equation}
    \min \limits_{\mathcal{G}} \, F(\mathbf{X}; \mathcal{G}) + \mathcal{L}_{\text{DAG}}(\mathcal{G}, \alpha_t, \mu_t), 
\end{equation}
where $\mathcal{L}_{\text{DAG}} = \alpha_t h(\mathcal{G}) + \frac{\mu_t}{2} |h(\mathcal{G})|^2$ is the penalty term in Lagrangian method. $\alpha_t$ and $\mu_t > 0$ are the penaly coefficients of the $t_{th}$ subproblem respectively.
Existing differentiable approaches always predefine a static $F(\mathbf{X};\mathcal{G})$ to measure SEMs in a fixed space (\eg penalized least-square loss in a fixed Euclidean Space~\cite{notears-mlp, notears2018, ng2019graph-autoencoder} and Evidence Lower Bound (ELBO) in a fixed asymmetry probability space~\cite{DAG-GNN-2019}) without considering graph itself. Considering that the score function aims to measure the goodness of a causal structure, a sufficient score function should include three parts: in addition to the first two terms in Equation~\eqref{eq: common-score-def} that have been well studied~\cite{goudet2018learning, RL-CD} but DAG-ness independent, a descriptor which encodes the structure's own information in $F$ is required.

\subsection{Proposed Model}
In this section, we will introduce the details of our model CASPER, which defines a DAG-ness aware score function in a dynamic causal metric space and reshape the optimization scheme for differentiable DAG structure learning. This allows the gradients of the loss function to be optimized towards the direction of more accurate causal graph reconstruction. 
For clarity, we first present the definition of our causal space and its desirable properties. 
Afterward, we describe the approach of applying it to DAG structure learning.
\newline

\noindent \textbf{Dynamic Causal Space.} 
Most of the standard score-based differentiable algorithms tend to apply the same score on different causal structures (see the example in Figure \ref{fig:teaser}), leading to suboptimal graph construction when using observational data.
As a result, the DAG learners are error-prone to constructing spurious edges due to DAG-ness independent forms and model vulnerability~\cite{kaiser2021unsuitability, he2021daring}. 
To address these problems, our goal is to discover a novel form of score function, predefined in a specific metric space (\ie causal space), which encodes the DAG-ness information into the score function for DAG-ness-aware causal structure learning. 
We introduce the CASPER framework, as shown in Figure~\ref{fig:pipeline}, which aims to adaptively perceive the causal structure and facilitate more accurate gradient optimization.
Before formally introducing the causal space, we first present the following lemma and definition of the Lipshitz norm for convenience in later notations.
Let $\mathbf{W} \in \{0, 1 \}^{d\times d}$ denote the $\mathcal{G}$'s adjacency  matrix. Specifically, $\mathbf{W}_{ij}=1$ if the directed edge $X_j \rightarrow X_i$ exists in $\mathcal{G}$, otherwise $\mathbf{W}_{ij} =0$. The DAG lemma is formulated as:
\begin{lemma}A matrix $\mathbf{W} \in \mathbb{R}^{d \times d}$ is a DAG if and only if
\begin{equation}
    h(\mathcal{G})=\operatorname{tr}\left(e^{\mathbf{W} \circ \mathbf{W}}\right)-d=0, 
    \label{eq: hw_2018}
\end{equation}
where $\circ$ is the Hadamard product.
\label{lemma: hw_2018}
\end{lemma}
Lemma~\ref{lemma: hw_2018}~\cite{notears2018} uses the trace of matrix exponential with Hadamard product of $\mathbf{W}$ to quantify the DAG-ness. 
To ease the numerical difficulty of computing $\operatorname{tr}\left(e^{\mathbf{W} \circ \mathbf{W}}\right)$, Yu et al.~\cite{DAG-GNN-2019} adopt a more convenient form of $h$ function:
\begin{equation}
    h(\mathcal{G})=\operatorname{tr}\left[\left(I+\alpha(\mathbf{W} \circ \mathbf{W})^d\right)\right]-d, \quad \alpha>0.
\end{equation}

\begin{definition}[Lipshitz norm]
   Let $\mathcal{M}_A$ and $\mathcal{M}_B$ be metric spaces. Let $\mathcal{T} : \mathcal{M}_A \rightarrow \mathcal{M}_B$ be a mapping function. The Lipshitz norm or Lipshitz modulus $||\mathcal{T}||_\text{Lip}$ of $\mathcal{T}$ is the supremum of the absolute difference quotients, \ie 
   \begin{equation}
\|\mathcal{T}\|_\text{Lip}:=\sup _{a \neq b, a, b \in \mathcal{M}_A} \frac{|\mathcal{T}(a)-\mathcal{T}(b)|}{\|a-b\|}.
\end{equation}
And we call the map $\mathcal{T}: \mathcal{M}_A \rightarrow \mathcal{M}_B$ Lipshitz continuous or imply Lipshitz if its Lipshitz norm is finite. 
\end{definition}
Guided by the aforementioned idea, we now formally give the definition of our dynamic causal space. 
\begin{definition}[Causal Space]
    Let $(\mathcal{S}, \mathcal{D})$ be a Polish space (complete metric space) for which Borel probability measure on $\mathcal{S}$ is a Radon measure. Let $\mathcal{P}(\mathcal{S})$ denote the collection of all probability measures $\nu$ on $\mathcal{S}$ with finite moment, that is, for any $z \in \mathcal{S}$, there exists some $z_0$ in $\mathcal{S}$ such that:
    \begin{equation}
\int_{\mathcal{S}} \mathcal{D}\left(z, z_0\right) d \nu(z)<\infty.
\end{equation}
For any $z_X, z_Y \in \mathcal{S}$ and let $P$ and $Q$ be the distribution of $z_X$ and $z_Y$. The distance in the space $\mathcal{S}$ between two probability measures $P$ and $Q$ in $\mathcal{P}(\mathcal{S})$ is defined as:
\begin{equation}
    \mathcal{D}_{\mathcal{S}}^{\mathcal{T}}(P, Q)=\sup _{\|\mathcal{T}\|_\text{Lip} \leq g(h(\mathcal{G}))}\left\{\mathbb{E}_{z_X \sim P}\left[\mathcal{T}\left(z_X\right)\right]-\mathbb{E}_{z_Y \sim Q}\left[\mathcal{T}\left(z_Y\right)\right]\right\},
\end{equation}
where $g$ is an increasing function, $h(\mathcal{G})$ is the DAG-ness function as explained in the part of Lemma~\ref{lemma: hw_2018}. $\mathcal{T}$ is a continuous mapping function $\mathcal{T}: \mathcal{S} \rightarrow \mathbb{R}$ and $||\cdot||_\text{Lip}$ is the Lipshitz norm. We call $\mathcal{S}$ \text{causal space} and $g(h(\mathcal{G}))$ structure-aware descriptor, which encodes the DAG-ness of the causal graph in causal space. And the distance $\mathcal{D}_{\mathcal{S}}^{\mathcal{T}}$ is called causal structure distance which defined in $\mathcal{S}$ with mapping function $\mathcal{T}$.
\label{def: causal space}
\end{definition}

Furthermore, our dynamic causal space has the following desirable properties:
\begin{itemize}
\item [(a)] The causal space $\mathcal{S}$ is a complete probability metric space that allows us to learn a DAG structure from a distributional view.
\item [(b)] DAG-ness information of a causal graph can be dynamically quantified by the smoothness of causal space through the process of structure learning.
\item [(c)] This space is noise-robust enough to observational data under "perturbation" (\ie additive noise).
\end{itemize}
Due to space limitations, we provide a detailed analysis and performance evaluation of these properties in the experiments presented in Section~\ref{sec:experiments}. 
Here, we offer some illustrative discussions regarding properties (a) to (c).
Property (a) implies that the causal structure distance defined in our causal space satisfies the axioms of a distance on Borel probability. 
This property allows us to capture the observational sampling distribution more faithfully to the DAG, particularly in real data, as demonstrated in our experiments.
Property (b) enables us to incorporate DAG-ness information into the score function during the optimization process, leading to more precise DAG solutions. This property enhances the optimization process and improves the accuracy of the resulting DAG.
Property (c) enhances the robustness of our model in handling heterogeneous data. 
Although there exists methods~\cite{gran-dag, DAG-GNN-2019} that achieve property (a) by measuring the SEMs from a static probabilistic view, they hardly satisfy (b) and (c).
These methods overlook the importance of dynamical structure information in the optimization and robustness of their models. 
Fortunately, our proposed dynamic causal space provides a comprehensive solution that satisfies all of the above properties.
\newline

\noindent \textbf{Learning DAG Structure in Causal Space.} Given the definition of causal space, we consider it a crucial criterion for differentiable score-based structure learning. Before delving into the learning process of DAG structures in the causal space, we provide the following characterization to guarantee the convergence of the structure learning process.
\begin{figure*}[!t]
    \centering
    \includegraphics[width=0.9\linewidth]{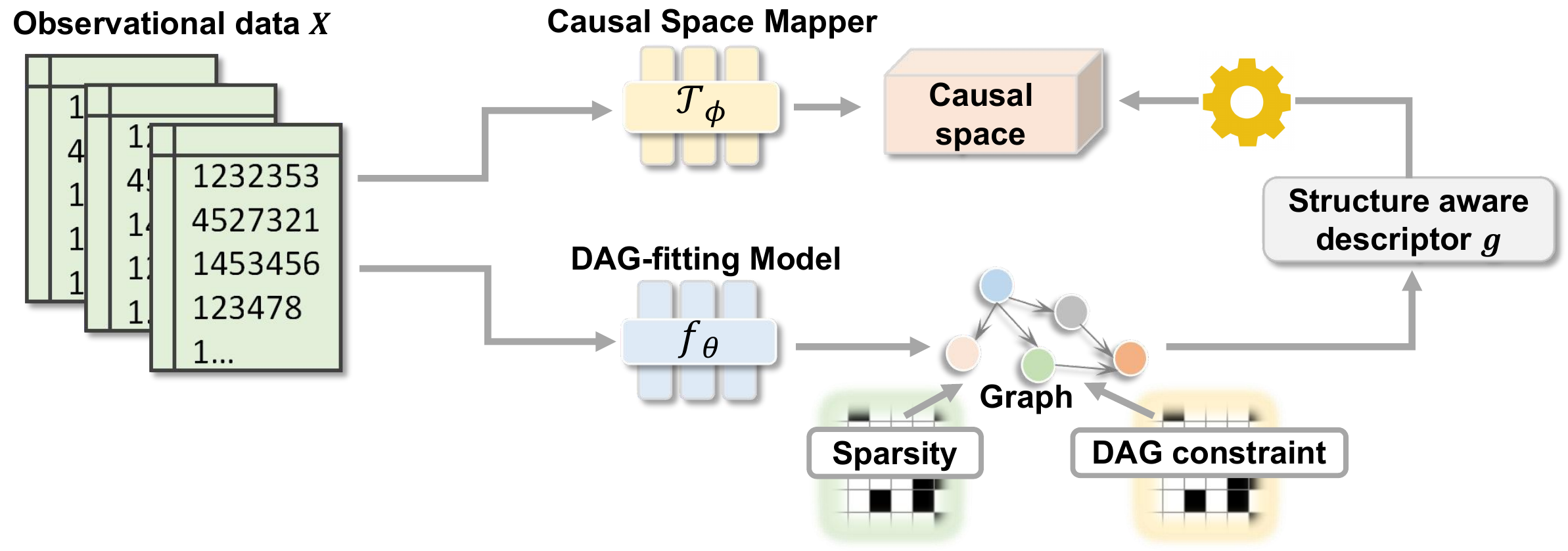}
    \caption{ Pipeline of Dynamic Causal Space for DAG Structure Learning (CASPER). Given observational data $\mathbf{X}$, we apply the causal space mapper $\mathcal{T}_\phi$ to encode data into causal space. Then we use the DAG-fitting model $f_\theta$ to optimize the causal graph with sparsity and DAG constraint. Finally, the DAG-ness information can be transmitted to the causal space through the structure-aware descriptor $g$. Thus the causal space is able to dynamically capture structural information and provide more accurate solutions.}
    \label{fig:pipeline}
\end{figure*}
\begin{algorithm}[ht]
   \caption{CASPER Algorithm for DAG Structure Learning}
   \label{alg:casper}
\begin{algorithmic}
   \STATE {\bfseries Input:} observational data $\mathbf{X} = \{\mathbf{x}^{(k)}\}_{k=1}^{n}$ sampled from $P_r$ and threshold $\omega >0$, maximum epoch in the inner loop $K_{\text{inner}}$, maximum epoch in the outer loop $K_{\text{outer}}$
   \STATE {\bfseries Initialize:} initialize the parameters of causal fitting model $\theta$ and parameters of causal space model $\phi$
   \FOR{$t=0 \text{ to }\tau_0$}
   \STATE Update $\theta$ and $\mathcal{G}$ to minimize $F_\phi$ and get $\mathcal{G}^{\text{pre}}$
   \ENDFOR
   \FOR{$k_1=0 \text{ to }K_{\text{outer}}$}
   \STATE Fix causal space model parameters $\phi$
   \STATE Calculate $F_\phi(\mathbf{X}; \mathcal{G}, \theta) + \mathcal{L}_{\text{DAG}}(\mathcal{G})$ in Equation~\ref{eq: bilevel}
   \STATE Update $\theta$ and $\mathcal{G}$ to minimize $F_\phi + \mathcal{L}_{\text{DAG}}$
   \FOR{$k_2=0 \text{ to }K_{\text{inner}}$}
   \STATE Fix graph $\mathcal{G}$ and the causal fitting model's parameters $\theta$  
   \STATE Update $\phi$ to maximize $F_\phi(\mathbf{X}; \mathcal{G}, \theta)$ in Equation~\ref{eq: new-score}
   \STATE $c \leftarrow \log(1+h(\mathcal{G}))$
   \STATE $\phi \leftarrow \text{clip}(\phi, -c, c)$
   \ENDFOR
   \ENDFOR
   \STATE Prune the edges less than $\omega$ of $\mathcal{G}$
   \STATE {\bfseries return} predicted $\mathcal{G}$
\end{algorithmic}
\end{algorithm}

\begin{proposition}[Convergence of Causal Space] 
Let $P$ be a distribution on our causal space $\mathcal{S}$ and $\{P_n\}_{n \in \mathbb{P}}$ be a sequence of distributions on $\mathcal{S}$. Then, considering limits as $n\rightarrow \infty$, $P_n \xrightarrow{distribution} P$ if and only if $\mathcal{D}(P_n, P)\rightarrow 0$ in $\mathcal{S}$, where $\xrightarrow{distribution}$ represents convergence in distribution for random variables.
\label{thm: convergence}
\end{proposition}

\begin{proof}
Let us start from a sequence $\{P_n\}$ such that $\mathcal{D}(P_n, P) \rightarrow 0$. Based on the definition~\ref{def: causal space} of Causal Space, for every $\mathcal{T}\in \text{Lip}_{g(h(\mathcal{G}))}$, we have $\int \mathcal{T}(P_n - P) \rightarrow 0$. And the same is true for any Lipshitz function. 
Then, we fix a subsequence $\{P_{n_k}\}$ that satisfies $\lim_k \mathcal{D}(P_{n_k})=\lim \sup_n \mathcal{D}(P_n, P)$. For each $k$, we pick a function $\mathcal{T}_{n_k}\in \text{Lip}_{g(h(\mathcal{G}))}$ such that $\int \mathcal{T}_{n_k}(P_{n_k}-P) = \mathcal{D}(P_{n_k}, P)$. Up to adding a constant, which does not affect the integral, we can assume that the $\mathcal{T}_{n_k}$ all vanish at the same point, and they are hence bounded and equicontinuous. By Arzela-Ascoli theorem~\cite{green1961arzela}, we can extract a sub-sequence uniformly converging to a certain $\mathcal{T} \in \text{Lip}_{g(h(\mathcal{G}))}$. By replacing the original subsequence with this new one, we now have:
\begin{equation}
    \mathcal{D}(P_{n_k}, P) = \int \mathcal{T}_{n_k}d(P_{n_k}-P) \rightarrow \int \mathcal{T}d(P-P)=0,
\end{equation}
where the convergence of the integral is justified by the distributional convergence $P_{n_k} \rightarrow P$ together with the strong convergence in continuous function $\mathcal{T}_{n_k} \rightarrow \mathcal{T}$. It shows that $\lim \sup_n \mathcal{D}(P_n, P) \rightarrow 0$ and concludes the proof.
\end{proof}

Proposition~\ref{thm: convergence} provides a good demonstration of the convergence in the causal space we proposed, which leads to theoretical guarantees for our optimization process. 

Formally, we cast the overall framework of CASPER to learn DAG structure in the causal space and boost causal discovery. Given observational data $\mathbf{X}$ sampled from $P_r$, the DAG-ness-aware score function defined in causal space $\mathcal{S}$ is:
\begin{equation}
    F_\phi(\mathbf{X}; \mathcal{G}, \theta) =  \left \{\mathbb{E}_{\mathbf{X}\sim P_r}[(\mathcal{T}_\phi(\mathbf{X}))] - \mathbb{E}_{\hat{\mathbf{X}}\sim P_\theta}[(\mathcal{T}_\phi(\hat{\mathbf{X}})] \right\} + \lambda \mathcal{R}_{\text{sparse}}(\mathcal{G}),
    \label{eq: new-score}
\end{equation}
where $\mathcal{T}_\phi$ is the causal space mapping function parameterized by $\phi$ and $\mathcal{R}_{\text{sparse}}(\mathcal{G})$ is the graph sparsity regularization term by $l_1$ norm in practice. $\hat{\mathbf{X}}$ is recovered through the data generative process of $\mathbf{X}$ by learnable DAG-fitting model $f$ with parameter set $\theta$, \ie $\hat{\mathbf{X}} = f(\mathbf{X}; \theta)$. 
Then we cast the overall framework of CASPER to learn DAG structure as the following bilevel optimization problem:
\begin{equation}
    \begin{array}{cc} 
& \min \limits_{\mathcal{G}, \theta} F_{\phi^{*}}(\mathbf{X}; \mathcal{G}, \theta) + \mathcal{L}_{\text{DAG}}(\mathcal{G}, \alpha_t, \mu_t)\\
\text { s.t. } & \phi^{*} \in 
{ \underset {\phi \in \mathcal{C}(\mathcal{G})} { \operatorname {arg\,max} } \, F_\phi(\mathbf{X}; \mathcal{G}, \theta) },
\end{array}
\label{eq: bilevel}
\end{equation}
where $\mathcal{C}(\mathcal{G}):=\left\{ \phi: \mathcal{T}_\phi \text{ is continuous, } ||\mathcal{T}_\phi||_{\text{Lip}} \leq g(h(\mathcal{G}))\right\}$ and $g(\cdot)$ is an increasing function which is $g(x) = \log (1+x)$ for implementation.
More specifically, Equation~\eqref{eq: bilevel} consists of two terms, where the inner-level objective (\ie optimize $\phi$ by maximizing $F_\phi$ to compute the causal structure distance in causal space) is nested within the outer-level objective (\ie optimize $\mathcal{G}$ and $\theta$ by minimizing the score function). We notice that solving the outer-level problem should be subject to the optimal value of the inner-level problem. For better convergence, we can pretrain the $\mathcal{G}$ and $\theta$ according to $F_\phi$ for a few epochs at first.

\begin{table*}[ht]
\centering
\caption{Linear Setting, for ER graphs of 10, 20, 50 nodes.}
\vspace{-10pt}
\resizebox{\textwidth}{!}{%
\begin{tabular}{cc|cccc|cccc|cccc}

\toprule
 &
   &
  \multicolumn{4}{c|}{10 nodes} &
  \multicolumn{4}{c|}{20 nodes} &
  \multicolumn{4}{c}{50 nodes} \\
\multicolumn{2}{c|}{ER2} & 
  TPR$\uparrow$ &
  FDR$\downarrow$ &
  SHD$\downarrow$ &
  SID$\downarrow$ &
  TPR$\uparrow$ &
  FDR$\downarrow$ &
  SHD$\downarrow$ &
  SID$\downarrow$ &
  TPR$\uparrow$ &
  FDR$\downarrow$ &
  SHD$\downarrow$ &
  SID$\downarrow$ \\ \midrule
\multicolumn{2}{c|}{Random} &
  0.08\footnotesize{$\pm$0.07} &
  0.93\footnotesize{$\pm$0.18} &
  33.2\footnotesize{$\pm$7.3} &
  95.6\footnotesize{$\pm$12.2} &
  0.11\footnotesize{$\pm$0.09} &
  0.89\footnotesize{$\pm$0.08} &
  56.8\footnotesize{$\pm$8.7} &
  292.3\footnotesize{$\pm$45.7} &
  0.04\footnotesize{$\pm$0.02} &
  0.90\footnotesize{$\pm$0.03} &
  397.3\footnotesize{$\pm$12.7} &
  1,082.0 \footnotesize{$\pm$182.2} \\
\multicolumn{2}{c|}{NOTEARS} &
  0.82\footnotesize{$\pm$0.07} &
  0.09\footnotesize{$\pm$0.05} &
  5.4\footnotesize{$\pm$1.6} &
  16.6\footnotesize{$\pm$5.8} &
  0.82\footnotesize{$\pm$0.09} &
  0.13\footnotesize{$\pm$0.04} &
  9.4\footnotesize{$\pm$4.1} &
  59.4\footnotesize{$\pm$10.7} &
  0.79\footnotesize{$\pm$0.06} &
  0.19\footnotesize{$\pm$0.03} &
  27.6\footnotesize{$\pm$7.7} &
  427.0\footnotesize{$\pm$186.1} \\
\multicolumn{2}{c|}{DAG-GNN} &
  0.83\footnotesize{$\pm$0.05} &
  0.12\footnotesize{$\pm$0.05} &
  4.8\footnotesize{$\pm$1.1} &
  12.9\footnotesize{$\pm$6.2} &
  0.83\footnotesize{$\pm$0.02} &
  0.13\footnotesize{$\pm$0.02} &
  8.7\footnotesize{$\pm$2.5} &
  48.5\footnotesize{$\pm$5.3} &
  0.81\footnotesize{$\pm$0.03} &
  0.13\footnotesize{$\pm$0.02} &
  24.3\footnotesize{$\pm$5.5} &
  334.2\footnotesize{$\pm$120.3} \\
\multicolumn{2}{c|}{NoCurl} &
  0.84\footnotesize{$\pm$0.04} &
  0.13\footnotesize{$\pm$0.03} &
  4.6\footnotesize{$\pm$1.3} &
  13.2\footnotesize{$\pm$5.1} &
  0.82\footnotesize{$\pm$0.05} &
  0.15\footnotesize{$\pm$0.05} &
  8.9\footnotesize{$\pm$3.4} &
  50.1\footnotesize{$\pm$6.7} &
  0.78\footnotesize{$\pm$0.07} &
  0.15\footnotesize{$\pm$0.03} &
  25.2\footnotesize{$\pm$6.0} &
  356.7\footnotesize{$\pm$165.2} \\
\multicolumn{2}{c|}{GraN-DAG} &
  0.82\footnotesize{$\pm$0.03} &
  0.08\footnotesize{$\pm$0.01} &
  5.2\footnotesize{$\pm$0.9} &
  14.8\footnotesize{$\pm$4.9} &
  0.80\footnotesize{$\pm$0.06} &
  0.14\footnotesize{$\pm$0.01} &
  8.5\footnotesize{$\pm$2.9} &
  47.2\footnotesize{$\pm$8.0} &
  0.82\footnotesize{$\pm$0.05} &
  0.12\footnotesize{$\pm$0.01} &
  24.8\footnotesize{$\pm$7.6} &
  289.1\footnotesize{$\pm$118.3} \\
\multicolumn{2}{c|}{DARING} &
  0.85\footnotesize{$\pm$0.02} &
  0.10\footnotesize{$\pm$0.01} &
  4.3\footnotesize{$\pm$1.7} &
  13.4\footnotesize{$\pm$4.5} &
  0.84\footnotesize{$\pm$0.05} &
  0.16\footnotesize{$\pm$0.02} &
  8.9\footnotesize{$\pm$3.0} &
  46.7\footnotesize{$\pm$6.5} &
  0.83\footnotesize{$\pm$0.06} &
  0.13\footnotesize{$\pm$0.02} &
  23.5\footnotesize{$\pm$6.2} &
  310.8\footnotesize{$\pm$159.6} \\
\multicolumn{2}{c|}{\textbf{CASPER(Ours)}} &
  \textbf{0.90\footnotesize{$\pm$0.04}} &
  \textbf{0.07\footnotesize{$\pm$0.02}} &
  \textbf{3.8\footnotesize{$\pm$0.8}} &
  \textbf{11.6\footnotesize{$\pm$4.3}} &
  \textbf{0.89\footnotesize{$\pm$0.09}} &
  \textbf{0.10\footnotesize{$\pm$0.03}} &
  \textbf{7.8\footnotesize{$\pm$3.7}} &
  \textbf{42.4\footnotesize{$\pm$7.2}} &
  \textbf{0.87\footnotesize{$\pm$0.05}} &
  \textbf{0.12\footnotesize{$\pm$0.03}} &
  \textbf{21.8\footnotesize{$\pm$5.8}} &
  \textbf{230.4\footnotesize{$\pm$119.8}} \\ \midrule \midrule
\multicolumn{2}{c|}{ER4} &
  TPR$\uparrow$ &
  FDR$\downarrow$ &
  SHD$\downarrow$ &
  SID$\downarrow$ &
  TPR$\uparrow$ &
  FDR$\downarrow$ &
  SHD$\downarrow$ &
  SID$\downarrow$ &
  TPR$\uparrow$ &
  FDR$\downarrow$ &
  SHD$\downarrow$ &
  SID$\downarrow$ \\ \midrule 
\multicolumn{2}{c|}{Random} &
  0.09\footnotesize{$\pm$0.17} &
  0.93\footnotesize{$\pm$0.09} &
  52.3\footnotesize{$\pm$16.7} &
  80.3\footnotesize{$\pm$17.7} &
  0.07\footnotesize{$\pm$0.03} &
  0.90\footnotesize{$\pm$0.08} &
  86.9\footnotesize{$\pm$7.0} &
  387.5\footnotesize{$\pm$52.3} &
  0.09\footnotesize{$\pm$0.08} &
  0.92\footnotesize{$\pm$0.08} &
  998.2\footnotesize{$\pm$45.9} &
  3,399.1\footnotesize{$\pm$489.2} \\
\multicolumn{2}{c|}{NOTEARS} &
  0.83\footnotesize{$\pm$0.06} &
  0.08\footnotesize{$\pm$0.03} &
  7.4\footnotesize{$\pm$2.7} &
  28.4\footnotesize{$\pm$5.8} &
  0.75\footnotesize{$\pm$0.01} &
  0.28\footnotesize{$\pm$0.05} &
  32.0\footnotesize{$\pm$5.4} &
  152.8\footnotesize{$\pm$27.0} &
  0.51\footnotesize{$\pm$0.12} &
  \textbf{0.27\footnotesize{$\pm$0.10}} &
  113.4\footnotesize{$\pm$29.5} &
  943.8\footnotesize{$\pm$172.2} \\
\multicolumn{2}{c|}{DAG-GNN} &
  0.82\footnotesize{$\pm$0.07} &
  0.12\footnotesize{$\pm$0.01} &
  7.0\footnotesize{$\pm$1.6} &
  29.4\footnotesize{$\pm$3.3} &
  0.81\footnotesize{$\pm$0.02} &
  0.25\footnotesize{$\pm$0.04} &
  29.5\footnotesize{$\pm$3.3} &
  138.4\footnotesize{$\pm$18.9} &
  0.55\footnotesize{$\pm$0.09} &
  0.28\footnotesize{$\pm$0.08} &
  115.2\footnotesize{$\pm$25.4} &
  835.3\footnotesize{$\pm$154.1} \\
\multicolumn{2}{c|}{NoCurl} &
  0.86\footnotesize{$\pm$0.10} &
  0.07\footnotesize{$\pm$0.02} &
  6.5\footnotesize{$\pm$2.3} &
  26.0\footnotesize{$\pm$4.9} &
  0.79\footnotesize{$\pm$0.03} &
  0.27\footnotesize{$\pm$0.03} &
  31.3\footnotesize{$\pm$2.1} &
  142.0\footnotesize{$\pm$14.9} &
  0.59\footnotesize{$\pm$0.10} &
  0.29\footnotesize{$\pm$0.06} &
  105.7\footnotesize{$\pm$26.2} &
  910.5\footnotesize{$\pm$129.0} \\
\multicolumn{2}{c|}{GraN-DAG} &
  0.84\footnotesize{$\pm$0.04} &
  0.06\footnotesize{$\pm$0.03} &
  7.8\footnotesize{$\pm$2.1} &
  25.5\footnotesize{$\pm$5.0} &
  0.78\footnotesize{$\pm$0.03} &
  0.26\footnotesize{$\pm$0.04} &
  29.7\footnotesize{$\pm$3.4} &
  143.5\footnotesize{$\pm$17.0} &
  0.52\footnotesize{$\pm$0.08} &
  0.31\footnotesize{$\pm$0.05} &
  110.3\footnotesize{$\pm$23.4} &
  854.3\footnotesize{$\pm$178.5} \\
\multicolumn{2}{c|}{DARING} &
  0.83\footnotesize{$\pm$0.06} &
  0.09\footnotesize{$\pm$0.01} &
  6.8\footnotesize{$\pm$1.8} &
  27.8\footnotesize{$\pm$3.5} &
  0.80\footnotesize{$\pm$0.02} &
  0.24\footnotesize{$\pm$0.02} &
  29.3\footnotesize{$\pm$2.0} &
  139.1\footnotesize{$\pm$15.4} &
  0.50\footnotesize{$\pm$0.12} &
  0.33\footnotesize{$\pm$0.05} &
  118.9\footnotesize{$\pm$27.0} &
  809.4\footnotesize{$\pm$165.3} \\
\multicolumn{2}{c|}{\textbf{CASPER(Ours)}} &
  \textbf{0.88\footnotesize{$\pm$0.05}} &
  \textbf{0.06\footnotesize{$\pm$0.04}} &
  \textbf{6.2\footnotesize{$\pm$2.1}} &
  \textbf{25.0\footnotesize{$\pm$2.7}} &
  \textbf{0.85\footnotesize{$\pm$0.03}} &
  \textbf{0.19\footnotesize{$\pm$0.02}} &
  \textbf{27.5\footnotesize{$\pm$2.9}} &
  \textbf{132.0\footnotesize{$\pm$16.3}} &
  \textbf{0.63\footnotesize{$\pm$0.10}} &
  0.29\footnotesize{$\pm$0.10} &
  \textbf{98.4\footnotesize{$\pm$31.1}} &
  \textbf{735.0\footnotesize{$\pm$160.2}} \\ \bottomrule
\end{tabular}%
}

\label{tab:linear-ER}
\end{table*}
Now we introduce how to solve the bilevel optimization in Equation~\eqref{eq: bilevel} in detail. In the inner loop, we fix the DAG-fitting model which predicts the data generative process of $\mathbf{X}$ and then update $\phi$ to maximize the score function $F_\phi$ to compute the causal structure distance in causal space $\mathcal{S}$. In the outer loop, upon the parameters of causal space mapping function $\phi$ is determined in the inner loop, we minimize the score function to optimize the DAG-fitting model. By alternately training the inner and outer loops, the score function can adaptively aware the causal structure in causal space, thus leading to more accurate gradient optimization and faster convergence to the optimal solution.
Our CASPER algorithm is summarized in the Algorithm~\ref{alg:casper}.

\section{Experiments} \label{sec:experiments}
In this section, we conduct extensive experiments to answer the research questions:
\begin{itemize}[leftmargin=*]
\item \textbf{RQ1: } How does CASPER perform compared to the previous methods in both linear and nonlinear settings?
\item \textbf{RQ2: } How do CASPER and other baselines perform with various factors (\ie noise scales, graph density)?  
\item \textbf{RQ2: } How does CASPER perform on real heterogeneous data compared with other applicable baselines?
\end{itemize}

\subsection{Experimental Settings}
\noindent \textbf{Baselines.} To answer the first and second question (RQ1 \& RQ2) , we select six state-of-the-art causal discovery methods as baselines for comparison:
\begin{itemize}[leftmargin=*]
    \item \textbf{NOTEARS~\cite{notears2018}} is specifically designed for linear settings and estimates the true causal graph by minimizing the fixed reconstruction loss with the continuous acyclicity constraint.
    \item \textbf{NOTEARS-MLP~\cite{notears-mlp}} is an extension of NOTEARS~\cite{notears2018} for nonlinear settings, which aims to approximate the generative structural equation model (\ie Equation~\eqref{eq: SEM}) by MLP while only utilizing the continuous acyclicity constraint to the first layer of the MLP.
    \item \textbf{DAG-GNN~\cite{DAG-GNN-2019}} reformulates DAG structure learning with variational autoencoder, where both encoder and decoder are graph neural networks. 
    By selecting the evidence lower bound as the score function, DAG-GNN is capable of effectively recovering the causal structure.
    \item \textbf{NoCurl~\cite{NoCurl}} utilizes a two-step procedure: initialize a cyclic solution first and then employ Hodge decomposition of graphs and learn a DAG structure by projecting the cyclic graph to the gradient of a potential function.
    \item \textbf{GraN-DAG~\cite{gran-dag}} adapts the constrained optimization formulation to allow for nonlinear relationships also by neural networks and makes use of the final pruning step to remove spurious edges.
    \item \textbf{DARING~\cite{he2021daring}} 
    introduces an adversarial learning strategy to impose an explicit residual independence constraint, aiming to improve the learning of acyclic graphs.
\end{itemize}
To comprehensively demonstrate the effectiveness of our proposed CASPER, extensive experiments are conducted with more baselines on the real heterogeneous data (RQ3). In addition to the baselines mentioned above, we further implement CD-NOD~\cite{heterogeneous-2020}, FGS~\cite{FGS-2017}, ICA-LiNGAM~\cite{ICA-LiNGAM} GOLEM~\cite{GOLEM}, and GES~\cite{GES} in the real-world benchmark dataset, \ie Sachs~\cite{sachs2005causal}. 
\newline 

\noindent \textbf{Hyperparameter Settings.} For linear settings, there are two main hyper-parameters, the sparsity coefficient $\lambda_1$ for the $l_1$-norm regularization term; $K_{\text{inner}}$ in Algorithm~\ref{alg:casper} for inner loops as we choose the same stop condition as NOTEARS~\cite{notears2018} to replace $K_\text{outer}$ for the parameter-free. We tune $\lambda_1$ in \{0.002, 0.005, 0.01, 0.015, 0.02, 0.03, 0.09, 0.1, 0.25 \} and tune $K_\text{inner}$ in \{1,2,3,4,5,6,7,8,9,10\}. For nonlinear settings, there are three main hyper-parameters in total: $\lambda_1, \lambda_2, K_\text{inner}$, among which $\lambda_1$ and $\lambda_2$ are for the $l_1$-norm and $l_2$-norm regularization terms respectively. And we follow the same tuning strategy in linear settings to tune the three hyper-parameters. We find that often $\lambda=0.01, K_\text{inner}=3$ wor well. In practice, we adopt multilayer perception (MLP) with parameters $\theta$ and $\phi$ to approximate $f_\theta$ and $\mathcal{T}_\phi$. More details of the network design will be open-sourced upon acceptance. As the training process is the augmented Lagrangian problem, we follow the same optimization of Lagrangian coefficients as NOTEARS for a fair comparison. Besides, 
following the convention in NOTEARS-based methods~\cite{notears2018, notears-mlp, he2021daring}, we adopt the same post-processing strategy for all the methods, cutting off the edges with values less than $0.3$. 
\newline

\noindent \textbf{Evaluation Metrics.} To evaluate the DAG structure learning, four metrics are reported: True Positive Rate (TPR), False Discovery Rate (FDR), Structural Hamming Distance (SHD), and Structural Intervention Distance (SID)~\cite{SID}, averaged over ten random trails. The SHD simply counts the number of missing, falsely detected, or reversed edges. And the SID is especially well suited for causal inference since it counts the number of couples $(i, j)$ such that the interventional distribution $p(x_j|do(X_i=\overline{x}))$ would be miscalculated if we use the estimated graph to form the parent adjustment set.
Higher TPR stands for better performance, while FDR, SHD, and SID should be lower to represent a better estimate of the target causal graph. 
\begin{table*}[!t]
\centering
\caption{Nonlinear Setting, for ER graphs of 10, 20, 50 nodes.}
\vspace{-10pt}
\resizebox{\textwidth}{!}{%
\begin{tabular}{cc|cccc|cccc|cccc}
\toprule
 &
   &
  \multicolumn{4}{c|}{10 nodes} &
  \multicolumn{4}{c|}{20 nodes} &
  \multicolumn{4}{c}{50 nodes} \\
\multicolumn{2}{c|}{ER2} & 
  TPR$\uparrow$ &
  FDR$\downarrow$ &
  SHD$\downarrow$ &
  SID$\downarrow$ &
  TPR$\uparrow$ &
  FDR$\downarrow$ &
  SHD$\downarrow$ &
  SID$\downarrow$ &
  TPR$\uparrow$ &
  FDR$\downarrow$ &
  SHD$\downarrow$ &
  SID$\downarrow$ \\ \midrule
\multicolumn{2}{c|}{Random} &
  0.06\footnotesize{$\pm$0.07} &
  0.94\footnotesize{$\pm$0.18} &
  35.2\footnotesize{$\pm$7.3} &
  95.6\footnotesize{$\pm$12.2} &
  0.08\footnotesize{$\pm$0.09} &
  0.89\footnotesize{$\pm$0.07} &
  59.8\footnotesize{$\pm$9.7} &
  392.3\footnotesize{$\pm$48.7} &
  0.04\footnotesize{$\pm$0.02} &
  0.92\footnotesize{$\pm$0.03} &
  486.3\footnotesize{$\pm$23.7} &
  1,134.2 \footnotesize{$\pm$210.3} \\
\multicolumn{2}{c|}{NOTEARS-MLP} &
  0.75\footnotesize{$\pm$0.12} &
  0.16\footnotesize{$\pm$0.09} &
  7.6\footnotesize{$\pm$2.3} &
  18.3\footnotesize{$\pm$9.1} &
  0.71\footnotesize{$\pm$0.12} &
  0.16\footnotesize{$\pm$0.08} &
  15.3\footnotesize{$\pm$6.1} &
  99.3\footnotesize{$\pm$18.4} &
  0.37\footnotesize{$\pm$0.03} &
  0.19\footnotesize{$\pm$0.07} &
  70.5\footnotesize{$\pm$8.7} &
  892.5\footnotesize{$\pm$146.4} \\
\multicolumn{2}{c|}{DAG-GNN} &
  0.81\footnotesize{$\pm$0.09} &
  0.14\footnotesize{$\pm$0.08} &
  7.0\footnotesize{$\pm$2.1} &
  14.1\footnotesize{$\pm$6.3} &
  0.78\footnotesize{$\pm$0.09} &
  0.12\footnotesize{$\pm$0.03} &
  10.1\footnotesize{$\pm$5.8} &
  80.3\footnotesize{$\pm$12.6} &
  0.41\footnotesize{$\pm$0.07} &
  0.23\footnotesize{$\pm$0.05} &
  59.2\footnotesize{$\pm$6.5} &
  698.4\footnotesize{$\pm$103.5} \\
\multicolumn{2}{c|}{NoCurl} &
  0.80\footnotesize{$\pm$0.07} &
  0.17\footnotesize{$\pm$0.07} &
  6.7\footnotesize{$\pm$2.4} &
  15.3\footnotesize{$\pm$5.0} &
  0.72\footnotesize{$\pm$0.12} &
  0.19\footnotesize{$\pm$0.03} &
  12.5\footnotesize{$\pm$4.3} &
  77.9\footnotesize{$\pm$12.3} &
  0.49\footnotesize{$\pm$0.05} &
  0.18\footnotesize{$\pm$0.06} &
  69.8\footnotesize{$\pm$7.4} &
  733.5\footnotesize{$\pm$130.4} \\
\multicolumn{2}{c|}{GraN-DAG} &
  0.83\footnotesize{$\pm$0.05} &
  0.12\footnotesize{$\pm$0.05} &
  5.1\footnotesize{$\pm$1.9} &
  11.5\footnotesize{$\pm$3.4} &
  0.81\footnotesize{$\pm$0.14} &
  0.16\footnotesize{$\pm$0.09} &
  9.9\footnotesize{$\pm$4.6} &
  65.4\footnotesize{$\pm$11.7} &
  0.52\footnotesize{$\pm$0.04} &
  0.14\footnotesize{$\pm$0.04} &
  52.8\footnotesize{$\pm$8.6} &
  635.4\footnotesize{$\pm$172.8} \\
\multicolumn{2}{c|}{DARING} &
  0.79\footnotesize{$\pm$0.09} &
  0.21\footnotesize{$\pm$0.03} &
  7.7\footnotesize{$\pm$3.1} &
  18.2\footnotesize{$\pm$4.8} &
  0.73\footnotesize{$\pm$0.07} &
  0.20\footnotesize{$\pm$0.06} &
  13.6\footnotesize{$\pm$5.2} &
  88.3\footnotesize{$\pm$24.4} &
  0.50\footnotesize{$\pm$0.07} &
  0.13\footnotesize{$\pm$0.05} &
  57.4\footnotesize{$\pm$9.3} &
  745.2\footnotesize{$\pm$120.6} \\
\multicolumn{2}{c|}{\textbf{CASPER(Ours)}} &
  \textbf{0.87\footnotesize{$\pm$ 0.13}} &
  \textbf{0.11\footnotesize{$\pm$0.07}} &
  \textbf{3.5\footnotesize{$\pm$1.8}} &
  \textbf{7.8\footnotesize{$\pm$4.1}} &
  \textbf{0.85\footnotesize{$\pm$0.08}} &
  \textbf{0.09\footnotesize{$\pm$0.04}} &
  \textbf{7.9\footnotesize{$\pm$1.5}} &
  \textbf{55.1\footnotesize{$\pm$9.8}} &
  \textbf{0.59\footnotesize{$\pm$0.06}} &
  \textbf{0.11\footnotesize{$\pm$0.03}} &
  \textbf{45.2\footnotesize{$\pm$6.0}} &
  \textbf{584.3\footnotesize{$\pm$102.7}} \\ \midrule \midrule
\multicolumn{2}{c|}{ER4} &
  TPR$\uparrow$ &
  FDR$\downarrow$ &
  SHD$\downarrow$ &
  SID$\downarrow$ &
  TPR$\uparrow$ &
  FDR$\downarrow$ &
  SHD$\downarrow$ &
  SID$\downarrow$ &
  TPR$\uparrow$ &
  FDR$\downarrow$ &
  SHD$\downarrow$ &
  SID$\downarrow$ \\ \midrule 
\multicolumn{2}{c|}{Random} &
  0.07\footnotesize{$\pm$0.16} &
  0.94\footnotesize{$\pm$0.09} &
  51.4\footnotesize{$\pm$15.7} &
  82.3\footnotesize{$\pm$17.7} &
  0.06\footnotesize{$\pm$0.04} &
  0.93\footnotesize{$\pm$0.18} &
  96.8\footnotesize{$\pm$6.9} &
  392.1\footnotesize{$\pm$42.3} &
  0.07\footnotesize{$\pm$0.07} &
  0.92\footnotesize{$\pm$0.05} &
  1,198.2\footnotesize{$\pm$54.2} &
  4,065.8\footnotesize{$\pm$584.2} \\
\multicolumn{2}{c|}{NOTEARS-MLP} &
  0.83\footnotesize{$\pm$0.14} &
  0.23\footnotesize{$\pm$0.03} &
  10.5\footnotesize{$\pm$1.9} &
  28.5\footnotesize{$\pm$11.2} &
  0.48\footnotesize{$\pm$0.09} &
  0.27\footnotesize{$\pm$0.05} &
  55.6\footnotesize{$\pm$9.3} &
  174.5\footnotesize{$\pm$32.1} &
  0.28\footnotesize{$\pm$0.08} &
  0.12\footnotesize{$\pm$0.06} &
  158.2\footnotesize{$\pm$10.4} &
  1,603.5\footnotesize{$\pm$88.9} \\
\multicolumn{2}{c|}{DAG-GNN} &
  0.87\footnotesize{$\pm$0.13} &
  0.18\footnotesize{$\pm$0.03} &
  6.8\footnotesize{$\pm$1.3} &
  18.7\footnotesize{$\pm$4.8} &
  0.52\footnotesize{$\pm$0.03} &
  0.21\footnotesize{$\pm$0.12} &
  49.2\footnotesize{$\pm$10.2} &
  150.3\footnotesize{$\pm$32.7} &
  0.43\footnotesize{$\pm$0.06} &
  0.15\footnotesize{$\pm$0.04} &
  150.4\footnotesize{$\pm$7.2} &
  1,536.9\footnotesize{$\pm$90.4} \\
\multicolumn{2}{c|}{NoCurl} &
  0.79\footnotesize{$\pm$0.08} &
  0.24\footnotesize{$\pm$0.06} &
  8.5\footnotesize{$\pm$3.5} &
  15.2\footnotesize{$\pm$7.7} &
  0.43\footnotesize{$\pm$0.05} &
  0.23\footnotesize{$\pm$0.06} &
  53.3\footnotesize{$\pm$8.4} &
  167.9\footnotesize{$\pm$34.4} &
  0.33\footnotesize{$\pm$0.04} &
  0.14\footnotesize{$\pm$0.06} &
  140.3\footnotesize{$\pm$7.6} &
  1,468.5\footnotesize{$\pm$100.2} \\
\multicolumn{2}{c|}{GraN-DAG} &
  0.90\footnotesize{$\pm$0.10} &
  \textbf{0.14\footnotesize{$\pm$0.02}} &
  6.4\footnotesize{$\pm$1.1} &
  5.8\footnotesize{$\pm$0.9} &
  0.47\footnotesize{$\pm$0.08} &
  0.25\footnotesize{$\pm$0.08} &
  47.5\footnotesize{$\pm$7.0} &
  149.8\footnotesize{$\pm$28.3} &
  0.42\footnotesize{$\pm$0.04} &
  \textbf{0.06\footnotesize{$\pm$0.03}} &
  128.6\footnotesize{$\pm$8.4} &
  1,232.4\footnotesize{$\pm$96.7} \\
\multicolumn{2}{c|}{DARING} &
  0.85\footnotesize{$\pm$0.07} &
  0.18\footnotesize{$\pm$0.09} &
  7.1\footnotesize{$\pm$1.6} &
  13.7\footnotesize{$\pm$5.9} &
  0.48\footnotesize{$\pm$0.07} &
  0.29\footnotesize{$\pm$0.10} &
  57.2\footnotesize{$\pm$4.6} &
  180.0\footnotesize{$\pm$43.5} &
  0.30\footnotesize{$\pm$0.05} &
  0.16\footnotesize{$\pm$0.05} &
  136.9\footnotesize{$\pm$12.5} &
  1,653.0\footnotesize{$\pm$78.4} \\
\multicolumn{2}{c|}{\textbf{CASPER(Ours)}} &
  \textbf{0.92\footnotesize{$\pm$0.06}} &
  0.15\footnotesize{$\pm$0.04} &
  \textbf{4.3\footnotesize{$\pm$2.1}} &
  \textbf{4.1\footnotesize{$\pm$1.1}} &
  \textbf{0.56\footnotesize{$\pm$0.04}} &
  \textbf{0.17\footnotesize{$\pm$0.09}} &
  \textbf{42.3\footnotesize{$\pm$5.6}} &
  \textbf{123.2\footnotesize{$\pm$24.5}} &
  \textbf{0.51\footnotesize{$\pm$0.03}} &
  0.08\footnotesize{$\pm$0.04} &
  \textbf{118.5\footnotesize{$\pm$8.0}} &
  \textbf{1,150.3\footnotesize{$\pm$70.2}} \\ \bottomrule
\end{tabular}%
}

\label{tab:nonlinear-ER}
\end{table*}
\subsection{Overall Performance Comparison (RQ1)} \label{sec: RQ1}
\textbf{Simulations.} Following the convention of causal discovery, the generating data differs along three dimensions: the number of nodes, the degree of edge sparsity, and the graph type. We consider two well-known graph sampling models, namely Erdos-Renyi (ER) and scale-free (SF)~\cite{ER-SF} 
with $kd$ expected edges (denoted as ER$k$ or $SFk$) and $d = \{10, 20, 50 \}$ nodes. Specifically, in linear settings, similar to Zheng et al.~\cite{notears2018} and Gao et al.~\cite{DAG-GAN}, the coefficients are assigned following \textit{Uniform} distribution $U(-2,-0.5) \, \cup \, U(0.5,2)$ with additive standard Gaussian noise. In nonlinear settings, same as Zheng et al.~\cite{notears-mlp}, we generate the ground truth structural equation model (SEM) in Equation~\eqref{eq: SEM} under the Gaussian process with radial basis function (RBF) kernel of bandwidth one, where $f_j(\cdot)$ is the additive noise model with $N_j$ as an i.i.d. random variable following the standard normal distribution. Notice that both of these settings are known to be fully identifiable~\cite{pena2018identifiability, peters2014causal-with-adm}. 
In this experiment, we explore the improvements when introducing both linear and nonlinear settings by comparing the DAG estimations against the ground truth structure. We simulate $\{ \text{ER2}, \text{ER4}, \text{SF2}, \text{SF4} \}$ graphs following ER or SF scheme with $d= \{ 10, 20 ,50\}$ nodes. For each graph, 10 datasets of 2,000 samples are generated and the mean and standard deviations (std) of the above metrics are reported for a fair comparison.
\newline 

\noindent \textbf{Results.} Table~\ref{tab:linear-ER}, Table~\ref{tab:nonlinear-ER}, and Tables in the Appendix demonstrate the comparison of overall performance on both linear and nonlinear synthetic data. Notice that the best-performing methods are bold and the error bars report the standard deviation across datasets over ten trials. We observe that: 
\begin{itemize}[leftmargin=*]
    \item \textbf{Our method CASPER significantly outperforms the state-of-the-art baselines across all datasets.} Specifically, our proposed model, \ie dynamic causal space can achieve consistent improvements in terms of SHD and SID, revealing a lower number of missing, falsely detected, reversed edges and a better estimation of the ground truth graph. We attribute the improvements to the dynamic and DAG-ness-aware causal space, which enhances the score function with adaptive attention to the causal graph and boosts the quality of score-based DAG structure learning. With a closer look at the TPR and FDR, CASPER typically lowers FDR by eliminating spurious edges and increases TPR by actively identifying more correct edges. This clearly demonstrates that CASPER effectively helps reach a more accurate gradient optimization through the structure distance in causal space, thus extracting better causal relationships. 

    \item \textbf{As the performance comparison among different graphs shows, the score-based methods suffer from a severe performance drop under high-dimensional graph data.} Despite the previous methods working well in linear and low-dimensional data, they fail to scale to more than 50 nodes in ER4 and SF4 graphs. Taking NOTEARS-MLP as an example, although it can achieve 83\% TPR in 10 nodes (ER4) of nonlinear settings, it suffers dramatic degradation, \ie only 28\% TPR in the 50 nodes (ER4) graph, which is mainly due to difficulties in enforcing acyclicity in high-dimensional dense graph data~\cite{wo-acyclicity, lippe2021efficient-cd}. However, our CASPER optimization model still performs well with TPR higher than 50\%, which shows the great potential of learning high dimensional and dense DAG structures under a DAG-ness aware optimization framework.
\end{itemize}


\subsection{Study of Various Factors (RQ2)} 
\noindent \textbf{Motivations.} In real-world applications, it is common to encounter graphs with various noise scales or different densities, where the underlying causal structure is invariant. We conjecture that a robust DAG structure learning framework is able to successfully estimate the graphs under various factors (\ie, noise scales and graph density). In this section, we discuss various factors that may affect the performance of CASPER and other methods.
\newline 

\noindent \textbf{Simulations.} 
We choose SF graphs with $d=20$ for the two case studies. Specifically, for different noise scales in both linear and nonlinear settings, we set the distribution of the noises as $N(\mu, 1), \mu \in \{0.2, 0.4, 0.6, 0.8, 1\}$ in SEMs of Equation~\eqref{eq: SEM} and choose SF2 graph to generate data. Following the settings in Section~\ref{sec: RQ1}, we set more graphs with various densities (\ie degree of nodes) from $\{2, 4, 6, 8, 10, 12\}$. For instance, the node degree $=10$ means there are $200$ edges in total when generating the SF graph.
\newline

\noindent \textbf{Results.} Figure~\ref{fig:effect_of_the_noisescale_of_the_graph} shows the evaluations with various noise scales and Figure~\ref{fig:effect_of_the_density_of_the_graph} reports the performance comparison with different density. Both empirical results of them are conducted on linear nonlinear synthetic SF datasets. Different colors separately refer to the state-of-the-art methods and our method in SHD performance. We find that:
\begin{figure}[!h]
    \centering
    \subfigure[Linear Setting]{\label{fig: linear-density1}\includegraphics[width=0.490\linewidth, height=3.1cm]{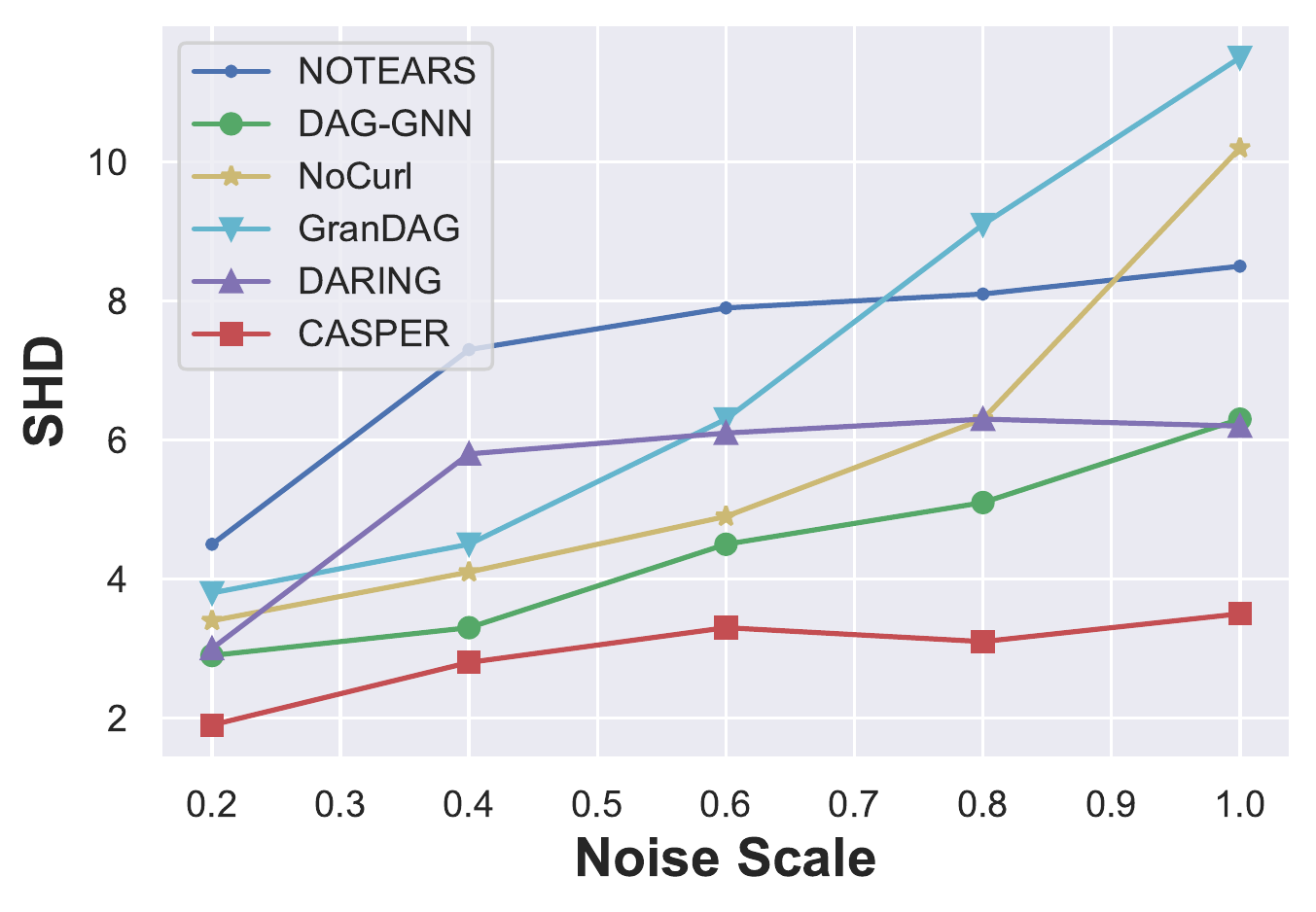}}
    \subfigure[Nonlinear Setting]{\label{fig:nonlinear-density1}\includegraphics[width=0.490\linewidth, height=3.1cm]{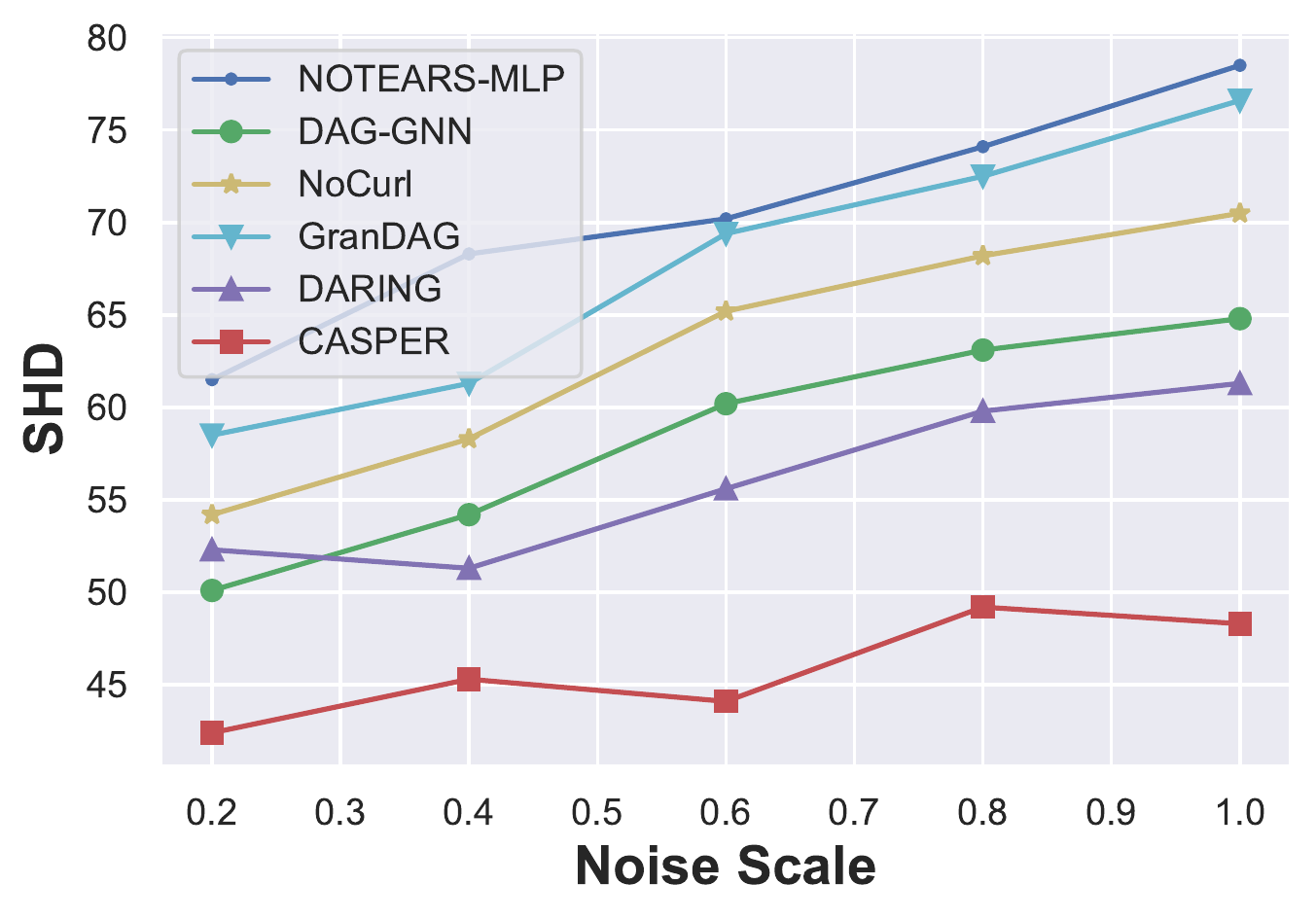}}
    \vspace{-10pt}
    \caption{SHD comparisons for various noise scales in SF2 graph with $20$ nodes.}
    \label{fig:effect_of_the_noisescale_of_the_graph}
\end{figure}

\begin{itemize}[leftmargin=*]
    \item \textbf{Compared with baselines, CASPER is noise-robust enough to observational data under various additive noise conditions.} Specifically, CASPER outperforms other methods consistently across all noise scale settings of SF graphs. We notice that other baselines are struggling from performance degradation when noise increases. We ascribe this hurdle to the static metric of score functions which is DAG-ness independent. In contrast, benefiting from DAG-ness aware score functions, our CASPER not only effectively captures the information from noise environments but also improves the DAG structure learning ability under perturbations.
    \item \textbf{As the performance comparison among density factors shows, our CASPER can better adapt to graphs with different degrees.} Although the baselines have the sparsity penalty to control the importance of graph density in regularization form, they do not perform well as our CASPER due to unawareness of causal structure in the score function.
    With a closer look at the evaluation curve of different densities, as the node degree increases, the improvements of CASPER over baselines get larger, which means CASPER could better adapt to the denser settings with adaptive structure attention.
\end{itemize}

\subsection{Evaluation on Real Data (RQ3)}
\noindent \textbf{Motivations.}
Heterogeneous data is a challenging yet frequently occurring issue in real-world observational data. Despite the variety of noise distribution, the underlying causal generating process always keeps stable in heterogeneous data. Specific DAG structure learners designed for heterogeneous data are prone to require prior knowledge of group annotations of each sample under strict conditions. However, group annotations are extremely costly and hard to collect and label.
\newline
\begin{figure}[!t]
    \centering
    \subfigure[Linear Setting]{\label{fig: linear-density}\includegraphics[width=0.490\linewidth, height=3.1cm]{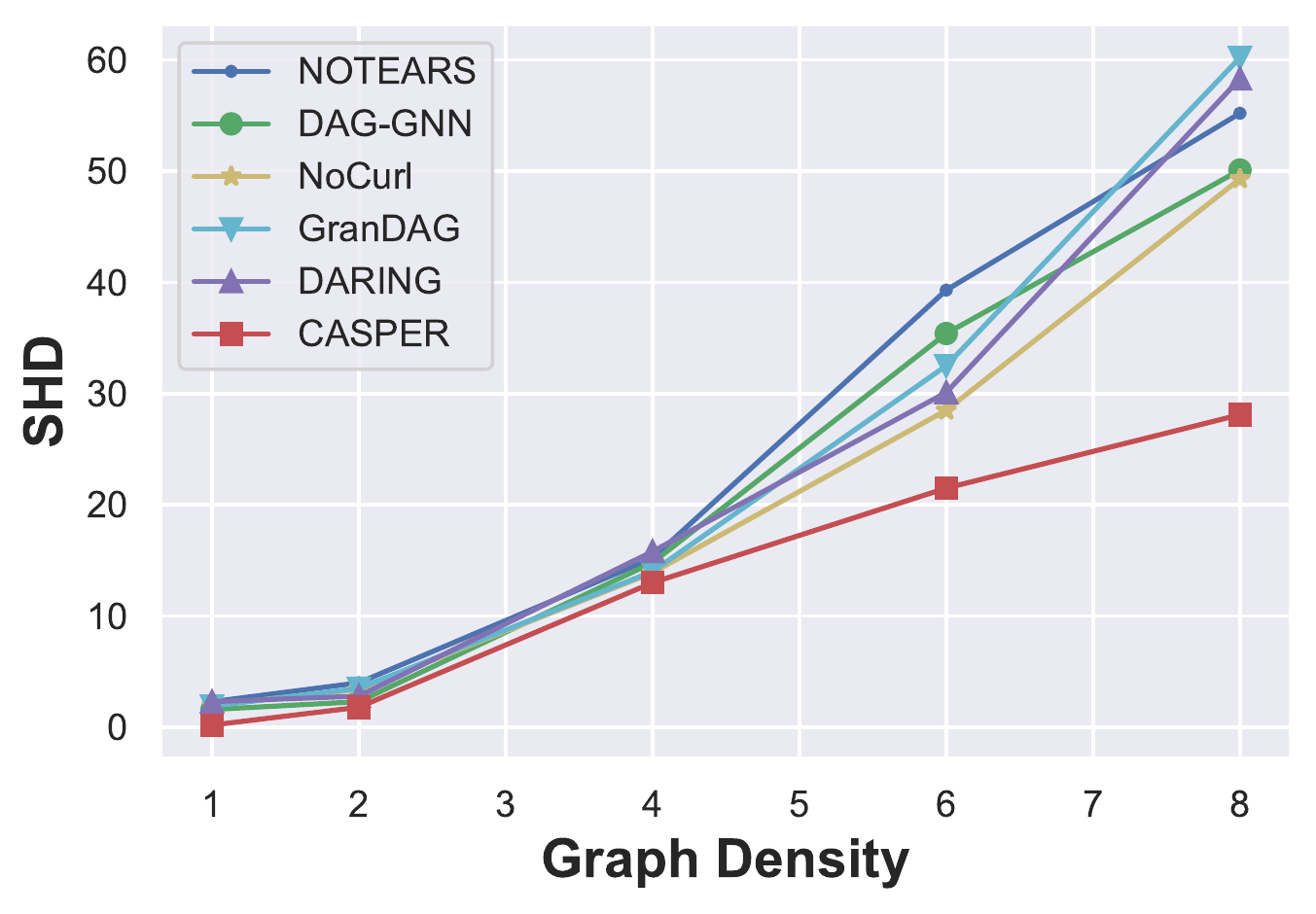}}
    \subfigure[Nonlinear Setting]{\label{fig:nonlinear-density}\includegraphics[width=0.490\linewidth, height=3.1cm]{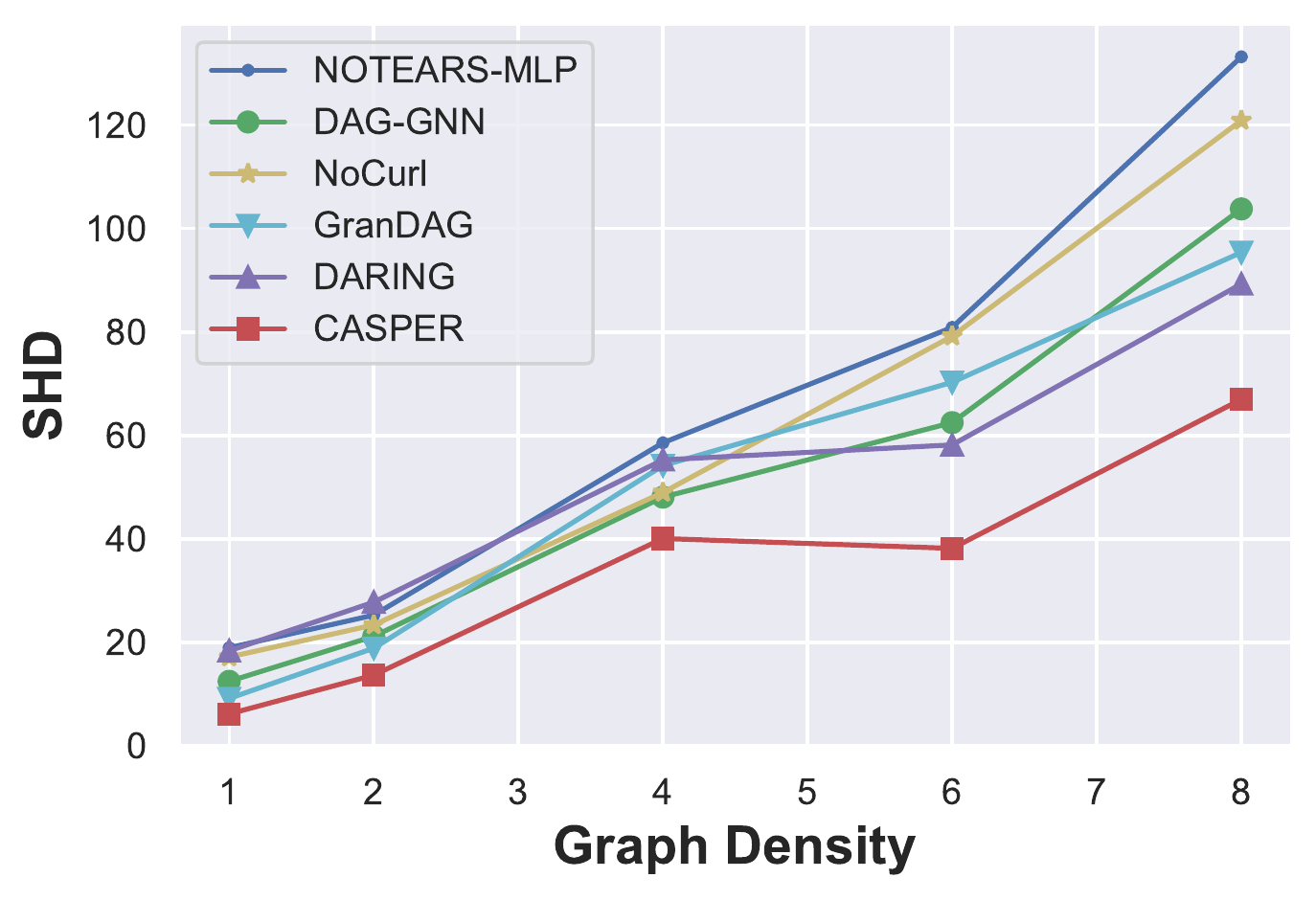}}
    \vspace{-10pt}
    \caption{SHD comparisons for different graph density conditions in SF graph with $20$ nodes.}
    \label{fig:effect_of_the_density_of_the_graph}
\end{figure}

\begin{table}[ht]
\centering
\caption{Empiricle results on Sachs~\cite{sachs2005causal} dataset.}
\vspace{-10pt}
\resizebox{\columnwidth}{!}{
\begin{tabular}{ccccc}
\toprule
Method                & \#Predicted Edges & \#Correct Edges & SHD         & SID         \\ \midrule
Random                & 22               & 1              & 23          & 63          \\
GES~\cite{GES}                   & 34               & 7              & 31          & 54          \\
ICA-LiNGAM~\cite{ICA-LiNGAM}           & 8                & 4              & 14          & 55          \\
FGS~\cite{FGS-2017}                   & 17               & 5              & 22          & 51          \\
GOLEM~\cite{GOLEM}                 & 6                & 4              & 15          & 53          \\
CD-NOD~\cite{heterogeneous-2020}                & 18               & 7              & 15          & -           \\
NOTEARS~\cite{notears2018}               & 20               & 6              & 17          & 48          \\
NOTEARS-MLP~\cite{notears-mlp}           & 19               & 7              & 16          & 45          \\
DAG-GNN~\cite{DAG-GNN-2019}               & 18               & 6              & 19          & 49          \\
DARING~\cite{he2021daring}                & 19               & 7              & 16          & 46          \\
NoCurl~\cite{NoCurl}                & 18               & 5              & 16          & 50          \\
GraN-DAG~\cite{gran-dag}              & 14               & 4              & 16          & 60          \\ \midrule
\textbf{CASPER(Ours)} & \textbf{15}      & \textbf{8}     & \textbf{12} & \textbf{42} \\ \bottomrule
\end{tabular}%
}

\label{tab:sachs}
\end{table}

\noindent \textbf{Dataset.} Sachs~\cite{sachs2005causal}, a real bioinformatics dataset, is for the discovery of the protein signaling network on expression levels of different proteins and phospholipids in human cells and is a popular benchmark for DAG structure learning, containing both observational and interventional data. Specifically, in Sachs, nine different perturbation conditions are imposed on sets of individual cells, each of which administers certain reagents to the cells. With the annotations of perturbation conditions, Sachs~\cite{sachs2005causal} is considered as the real-world heterogeneous dataset~\cite{mooij2020joint}. The true graph from~\cite{sachs2005causal} containing 11 nodes and 17 edges on 7,466 samples is widely used for research on graphical models, with experimental annotations accepted by the biological research community.
\newline

\noindent \textbf{Results.} In this benchmark dataset, we compare with recent continuous score-based methods, including NOTEARS~\cite{notears2018}, NOTEARS-MLP~\cite{notears-mlp}, DAG-GNN~\cite{DAG-GNN-2019}, NoCurl~\cite{NoCurl}, GraN-DAG~\cite{gran-dag}, DARING~\cite{he2021daring} and GOLEM~\cite{GOLEM}, traditional structural causal models ICA-LiNGAM~\cite{ICA-LiNGAM}, and combinational methods GES~\cite{GES} and FGS~\cite{FGS-2017}. Because the true causal graph in Sachs is sparse that a purely empty graph can reach as low as 17 in SHD, we report the \#total predicted edges, \#correct edges, SHD and SID in Table~\ref{tab:sachs}.

As Table~\ref{tab:sachs} illustrates, CASPER drives great performance breakthroughs and outperforms all other methods in correct discovery of the ground truth on real heterogeneous data. Specifically, most previous methods (\eg NOTEARS-based, GOLEM) suffer from notorious performance drops when the homogeneity assumption is unsatisfied, and pose hurdles from being scaled up to real-world large-scale applications. In stark contrast, benefiting from DAG-ness aware attention to the causal graph, 
CASPER achieves lower SHD as well as SID and improves the predicted correct edges, which accomplishes more profound causation understanding, leading to higher DAG structure learning quality. This validates that the potential of CASPER as a promising research direction for enhancing robustness and generalization for DAG structure learning when encountering various real-world data.

\section{related work}
DAG structure learning has recently taken the field of machine learning by storm~\cite{Causality-for-machine-learning}. A DAG $\mathcal{G}$ and a joint distribution are faithful to each other if and only if the conditional independencies true in the joint distribution are entailed by $\mathcal{G}$~\cite{pearl1988probabilistic}. This \ff{principle of faithfulness} enables one to recover $\mathcal{G}$ from the joint distribution. Given i.i.d. samples $\mathbf{X}$ from an unknown distribution corresponding to a faithful but unknown causal graph, \textit{DAG structure learning} refers to recovering the causal graph from $\mathbf{X}$.
In this section, we review the works of some related fields with this work. 

Generally speaking, \ff{there are two primary classes of algorithms employed for DAG structure learning (\ie causal discovery): constraint-based methods and score-based methods}. Our CASPER falls into the second class.

Constraint-based causal discovery methods first \ff{apply} conditional independence tests to \ff{identify the} causal skeleton under \ff{a} faithfulness assumption. Then they \ff{establish} the orientations of edges up to the Markov equivalence class, which \ff{usually contains structurally diverse DAGs with potentially unoriented edges.}
Examples include~\cite{sun2007kernel, zhang2012kernel} that use kernel-based conditional independence criteria and the well-known PC algorithm~\cite{spirtes2000causation} which implements the independence tests when no unobserved confounder exists. \ff{In scenarios involving} unobserved confounders, \ff{the} fast causal inference algorithm (FCI)~\cite{FCI} also calls independence judgement like PC, but targets an extended causal graph with \ff{bi-directed} edges. However, these methods are not robust as small errors in building the graph skeleton or are limited by sample size, thus leading to notorious performance degradation in the inferred Markov equivalence class. To alleviate the drawbacks, some score-based methods~\cite{GES, FGS-2017} have been proposed as an alternative solution. 

Score-based methods~\cite{prob-graph-model-2009, element-of-causal-inference} cast the problem of structure learning as an optimization problem over the space of DAGs.
Many popular methods tackle the combinatorial nature of the problem by performing the form of greedy search.
The Greedy Equivalence Search (GES)~\cite{GES} and its extension FGS~\cite{FGS-2017} utilizes a score function called BDeu to measure the correctness of the conditional independence of the target graph. The discrete algorithm starts with an empty graph and adopts a greedy strategy to change edges until the convergence of the score. In contrast to the methods that only identify the Markov equivalence class, SEMs, a class of score-based methods, can determine the true causal graph from the same equivalence class under additional assumptions. For instance, PNL~\cite{PNL} demonstrates its definite identifiability in two-variable settings except 5 special cases by examing if the disturbance is independent. On the other hand, conventional approaches, such as LiNGAM~\cite{ICA-LiNGAM}, combinatorially search for the DAG structure for multiple variables by converting the topological ordering of the causality diagram into the lower triangular matrix.

However, learning the DAG structure from purely observational data \ff{remains a challenge} mainly due to the intractable combinatorial nature of acyclic graph space~\cite{NP-hard, chickering2004large-nphard, chickering1996learning-nphard}. Fortunately, a recent breakthrough, NOTEARS proposed in Zheng et al.~\cite{notears2018}, reformulates the discrete DAG constraint into a continuous equality constraint, resulting in a differentiable score-based optimization problem. Further, there are various subsequent works after NOTEARS. \ff{DAG-GNN~\cite{DAG-GNN-2019} proposes a variant of gradient-optimized formulation in autoencoder architecture; NOTEARS-MLP~\cite{notears-mlp} and GraN-DAG~\cite{gran-dag} extend the NOTEARS framework to deal with more nonlinear functions using neural networks; RL-BIC~\cite{RL-CD} introduces reinforcement learning (RL) to search for the DAG; GOLEM~\cite{GOLEM} utilizes a likelihood-based objective with soft sparsity and DAG constraints instead of constrained optimization. In addition to single domain exploration}, some researchers~\cite{huang2020causal-cdnod, heterogeneous-2020, rescore} study causal discovery on multi-domain (\ie heterogeneous data where the underlying causal generating process remains stable but the noise distributions may vary). In this paper, we mainly focus on differentiable score-based DAG structure learning.
\section{Conclusion}
Despite the great success of causal structure learning on synthetic data, today's differentiable causal discovery methods are still far from being able to recover the target causal structures in various real-world applications. 
In this paper, we proposed CASPER, an effective optimization framework that boosts the DAG structure learning in a dynamic causal space, which adaptively perceives the graph structure during the training process.
Grounded by empirical visualization studies, CASPER is noise-robust to observational data under perturbation. 
Extensive experiments demonstrate that the remarkable improvement of CASPER on a variety of synthetic and real heterogeneous datasets indeed comes from the DAG-ness aware score function. 

One limitation of CASPER is that our framework is built on differentiable score-based causal discovery. 
In the future study, we will explore similar DAG-ness-aware strategies in more general structure learning frameworks. 
We believe that our CASPER provides a promising research direction to diagnose the performance degradation for nonlinear and noise data in DAG structure learning, and will inspire more valuable works for learning accurate causal graphs from observational data.

\section*{Acknowledgments}
This research is supported by the National Natural Science Foundation of China (9227010114) and the University Synergy Innovation Program of Anhui Province (GXXT-2022-040).

\bibliographystyle{ACM-Reference-Format}
\bibliography{kdd2023_conference}

\appendix
\section{appendix}
\subsection{Additional Experiments}
More experimental results on both the linear and nonlinear synthetic data are reported in Appendix as Table~\ref{tab:linear-SF} and ~\ref{tab:nonlinear-SF} shows.

In order to further show the efficiency of our algorithm CASPER, we conduct additional experiments as Table~\ref{tab:time1}, ~\ref{tab:time2}, and ~\ref{tab:time3} shows. As the both synthetic and real data show, our CASPER only adds a negligible amount of computational time cost but achieves significant performance improvement compared to NOTEARS-based methods. 

\subsection{More Discussion}
\noindent \textbf{Discussion of dynamic causal space:} The current score function, a measure to evaluate candidate-directed graphs in structure learning, solely takes into account the data fitness while neglecting the graph structure. However, there is an implicit assumption for the score function to appropriately evaluate, the estimated directed graph must be an acyclic graph throughout the training phase, which is not achievable. We, therefore, believe that the next-generation score function for structure learning should account for both data fitness and graph structure.

In mathematics, a "metric space" is a set with a notion of distance between its elements. The distance is measured by a function called the metric or distance function. In structure learning, the space is equipped with a set of directed graphs and the notion of distance between candidate graphs is the predefined score function. Since our metric (causal distance) is dynamically changed according to the directed graphs, we defined it as "dynamic causal space" (Definition~\ref{def: causal space} in the paper). We would like to highlight here that defining a dynamic measure by incorporating knowledge of time, geometry, and data-related information has been explored in many other fields~\cite{gu2017newtype, duchi2011adaptive, hazan2016convex}.

The "dynamic causal space" proposed in this paper is one of the potential solutions that fuse structural information (DAG-ness) into score function. Here "dynamic" signifies the incorporation of different Lipshitz constants in the score function, which causes the goodness-of-fit to vary as the DAGness changes. By doing this, we may dynamically adjust the metric (causal distance in our paper) in the graph spaces as opposed to using a "static" measure neglecting the DAG-ness of the graph. Let us consider a simple scenario: If the initial graph's DAG-ness is poor (\ie h(G) in our paper is high), we employ a more complex function (with a larger Lipshitz norm) to measure its distance to the true graph. As the optimization progresses and the graph's DAG-ness improves, we switch to simpler functions (with a smaller Lipshitz norm) for measuring the distance. This adaptive adjustment allows us to better optimize the graph and avoid local optima. Notably, if the initial graph is already the true graph, the Lipshitz constant would be zero.

\noindent \textbf{Discussion of motivation example:} Considering the experiments on linear models in Figure~\ref{fig:teaser}, we would like to emphasize that our intention was not to carefully design a linear case to create such a difficult situation for NOTEARS~\cite{notears2018}. In contrast, we aim to demonstrate that even in relatively simple cases, the static measure (least square in NOTEARS~\cite{notears2018}) might not perform well. To show our motivation thoroughly, we have also conducted a nonlinear case in Table~\ref{tab:nonlinear-teaser}. The true graph is generated from:
$A:=\epsilon_A(\sim U(-1,1)), B:=2 \sin (A)+\epsilon_B(\sim N(0,2)), C:=\cos (A)+0.5 \sin (B)+\epsilon C(\sim N(-1,1)), D:=0.5 C+\epsilon D(\sim U(0,1))$. We also capture the three phases in the optimization process.

\begin{table}[!h]
\vspace{-3mm}
\caption{Nonlinear model experiments for Figure~\ref{fig:teaser}.}
\centering
\resizebox{0.8\columnwidth}{!}{%
\begin{tabular}{cccc}
\hline
               & Phase-1 & Phase-2 & Phase-3 \\ \hline
h(W)           & 20.13   & 10.32   & 3.91    \\
Score(NOTEARS) & 10.20   & 10.43   & 9.87    \\
Score(CASPER)  & 10.53   & 6.04    & 1.55    \\ \hline
\end{tabular}%
}

\label{tab:nonlinear-teaser}
\end{table}

\begin{table*}[b]
\centering
\caption{Linear Setting, for SF graphs of 10, 20, 50 nodes.}
\resizebox{\textwidth}{!}{%
\begin{tabular}{cc|cccc|cccc|cccc}
\toprule
 &
   &
  \multicolumn{4}{c|}{10 nodes} &
  \multicolumn{4}{c|}{20 nodes} &
  \multicolumn{4}{c}{50 nodes} \\
\multicolumn{2}{c|}{SF2} & 
  TPR$\uparrow$ &
  FDR$\downarrow$ &
  SHD$\downarrow$ &
  SID$\downarrow$ &
  TPR$\uparrow$ &
  FDR$\downarrow$ &
  SHD$\downarrow$ &
  SID$\downarrow$ &
  TPR$\uparrow$ &
  FDR$\downarrow$ &
  SHD$\downarrow$ &
  SID$\downarrow$ \\ \midrule
\multicolumn{2}{c|}{Random} &
  0.15\footnotesize{$\pm$0.03} &
  0.91\footnotesize{$\pm$0.09} &
  32.2\footnotesize{$\pm$8.0} &
  35.1\footnotesize{$\pm$7.3} &
  0.11\footnotesize{$\pm$0.10} &
  0.89\footnotesize{$\pm$0.03} &
  43.2\footnotesize{$\pm$5.4} &
  96.8\footnotesize{$\pm$10.4} &
  0.10\footnotesize{$\pm$0.08} &
  0.89\footnotesize{$\pm$0.07} &
  334.2\footnotesize{$\pm$16.9} &
  993.3 \footnotesize{$\pm$145.4} \\
\multicolumn{2}{c|}{NOTEARS} &
  0.94\footnotesize{$\pm$0.10} &
  0.11\footnotesize{$\pm$0.03} &
  0.9\footnotesize{$\pm$0.4} &
  1.0\footnotesize{$\pm$0.5} &
  0.90\footnotesize{$\pm$0.06} &
  0.02\footnotesize{$\pm$0.01} &
  4.0\footnotesize{$\pm$1.9} &
  9.8\footnotesize{$\pm$4.8} &
  0.82\footnotesize{$\pm$0.03} &
  0.07\footnotesize{$\pm$0.05} &
  23.6\footnotesize{$\pm$6.2} &
  75.4\footnotesize{$\pm$17.5} \\
\multicolumn{2}{c|}{DAG-GNN} &
  1.00\footnotesize{$\pm$0.00} &
  0.00\footnotesize{$\pm$0.00} &
  0.0\footnotesize{$\pm$0.0} &
  0.0\footnotesize{$\pm$0.0} &
  0.93\footnotesize{$\pm$0.05} &
  0.01\footnotesize{$\pm$0.01} &
  2.3\footnotesize{$\pm$0.6} &
  9.1\footnotesize{$\pm$3.7} &
  0.90\footnotesize{$\pm$0.04} &
  \textbf{0.04\footnotesize{$\pm$0.03}} &
  18.5\footnotesize{$\pm$8.2} &
  54.2\footnotesize{$\pm$10.3} \\
\multicolumn{2}{c|}{NoCurl} &
  0.95\footnotesize{$\pm$0.01} &
  0.08\footnotesize{$\pm$0.02} &
  0.8\footnotesize{$\pm$0.1} &
  0.9\footnotesize{$\pm$0.1} &
  0.90\footnotesize{$\pm$0.03} &
  0.03\footnotesize{$\pm$0.02} &
  3.4\footnotesize{$\pm$0.5} &
  8.9\footnotesize{$\pm$5.0} &
  0.91\footnotesize{$\pm$0.01} &
  0.06\footnotesize{$\pm$0.01} &
  16.7\footnotesize{$\pm$6.5} &
  60.8\footnotesize{$\pm$8.9} \\
\multicolumn{2}{c|}{Gran-DAG} &
  0.91\footnotesize{$\pm$0.02} &
  0.12\footnotesize{$\pm$0.04} &
  0.9\footnotesize{$\pm$0.0} &
  0.4\footnotesize{$\pm$0.1} &
  0.92\footnotesize{$\pm$0.04} &
  0.01\footnotesize{$\pm$0.02} &
  3.5\footnotesize{$\pm$1.4} &
  9.3\footnotesize{$\pm$4.2} &
  0.75\footnotesize{$\pm$0.02} &
  0.12\footnotesize{$\pm$0.03} &
  19.1\footnotesize{$\pm$7.3} &
  53.1\footnotesize{$\pm$13.4} \\
\multicolumn{2}{c|}{DARING} &
  0.96\footnotesize{$\pm$0.01} &
  0.07\footnotesize{$\pm$0.01} &
  0.5\footnotesize{$\pm$0.1} &
  0.6\footnotesize{$\pm$0.2} &
  0.93\footnotesize{$\pm$0.01} &
  0.09\footnotesize{$\pm$0.01} &
  2.8\footnotesize{$\pm$0.9} &
  8.3\footnotesize{$\pm$4.0} &
  0.89\footnotesize{$\pm$0.03} &
  0.04\footnotesize{$\pm$0.01} &
  18.9\footnotesize{$\pm$5.4} &
  58.0\footnotesize{$\pm$12.7} \\
\multicolumn{2}{c|}{\textbf{CASPER(Ours)}} &
  \textbf{1.00\footnotesize{$\pm$ 0.00}} &
  \textbf{0.00\footnotesize{$\pm$0.00}} &
  \textbf{0.0\footnotesize{$\pm$0.0}} &
  \textbf{0.0\footnotesize{$\pm$0.0}} &
  \textbf{0.96\footnotesize{$\pm$0.03}} &
  \textbf{0.00\footnotesize{$\pm$0.03}} &
  \textbf{1.8\footnotesize{$\pm$0.7}} &
  \textbf{7.4\footnotesize{$\pm$2.6}} &
  \textbf{0.96\footnotesize{$\pm$0.02}} &
  0.05\footnotesize{$\pm$0.02} &
  \textbf{13.4\footnotesize{$\pm$9.3}} &
  \textbf{43.2\footnotesize{$\pm$14.9}} \\ \midrule \midrule
\multicolumn{2}{c|}{SF4} &
  TPR$\uparrow$ &
  FDR$\downarrow$ &
  SHD$\downarrow$ &
  SID$\downarrow$ &
  TPR$\uparrow$ &
  FDR$\downarrow$ &
  SHD$\downarrow$ &
  SID$\downarrow$ &
  TPR$\uparrow$ &
  FDR$\downarrow$ &
  SHD$\downarrow$ &
  SID$\downarrow$ \\ \midrule 
\multicolumn{2}{c|}{Random} &
  0.13\footnotesize{$\pm$0.01} &
  0.93\footnotesize{$\pm$0.15} &
  57.2\footnotesize{$\pm$10.3} &
  79.1\footnotesize{$\pm$8.7} &
  0.09\footnotesize{$\pm$0.05} &
  0.88\footnotesize{$\pm$0.05} &
  108.2\footnotesize{$\pm$12.9} &
  155.6\footnotesize{$\pm$37.2} &
  0.12\footnotesize{$\pm$0.11} &
  0.89\footnotesize{$\pm$0.04} &
  1,023.5\footnotesize{$\pm$49.5} &
  1,903.9\footnotesize{$\pm$194.3} \\
\multicolumn{2}{c|}{NOTEARS} &
  0.94\footnotesize{$\pm$0.03} &
  0.03\footnotesize{$\pm$0.02} &
  12.2\footnotesize{$\pm$1.2} &
  6.2\footnotesize{$\pm$3.3} &
  0.90\footnotesize{$\pm$0.05} &
  0.12\footnotesize{$\pm$0.06} &
  15.2\footnotesize{$\pm$7.0} &
  28.6\footnotesize{$\pm$10.3} &
  0.81\footnotesize{$\pm$0.18} &
  0.25\footnotesize{$\pm$0.07} &
  77.6\footnotesize{$\pm$36.9} &
  176.2\footnotesize{$\pm$49.0} \\
\multicolumn{2}{c|}{DAG-GNN} &
  0.92\footnotesize{$\pm$0.02} &
  0.02\footnotesize{$\pm$0.01} &
  7.6\footnotesize{$\pm$2.0} &
  5.4\footnotesize{$\pm$1.8} &
  0.91\footnotesize{$\pm$0.03} &
  \textbf{0.10\footnotesize{$\pm$0.05}} &
  14.8\footnotesize{$\pm$6.3} &
  21.3\footnotesize{$\pm$11.6} &
  0.83\footnotesize{$\pm$0.14} &
  0.21\footnotesize{$\pm$0.09} &
  70.1\footnotesize{$\pm$12.3} &
  146.7\footnotesize{$\pm$50.2} \\
\multicolumn{2}{c|}{NoCurl} &
  0.93\footnotesize{$\pm$0.01} &
  0.05\footnotesize{$\pm$0.03} &
  5.9\footnotesize{$\pm$1.5} &
  6.1\footnotesize{$\pm$2.4} &
  0.89\footnotesize{$\pm$0.05} &
  0.12\footnotesize{$\pm$0.05} &
  13.9\footnotesize{$\pm$6.0} &
  22.5\footnotesize{$\pm$6.7} &
  0.80\footnotesize{$\pm$0.08} &
  0.18\footnotesize{$\pm$0.05} &
  75.4\footnotesize{$\pm$17.0} &
  164.0\footnotesize{$\pm$43.1} \\
\multicolumn{2}{c|}{Gran-DAG} &
  0.94\footnotesize{$\pm$0.02} &
  \textbf{0.01\footnotesize{$\pm$0.02}} &
  5.3\footnotesize{$\pm$1.3} &
  5.8\footnotesize{$\pm$3.0} &
  0.92\footnotesize{$\pm$0.02} &
  0.13\footnotesize{$\pm$0.07} &
  14.0\footnotesize{$\pm$7.1} &
  21.0\footnotesize{$\pm$9.7} &
  0.83\footnotesize{$\pm$0.12} &
  \textbf{0.13\footnotesize{$\pm$0.03}} &
  67.5\footnotesize{$\pm$15.9} &
  159.7\footnotesize{$\pm$36.9} \\
\multicolumn{2}{c|}{DARING} &
  0.91\footnotesize{$\pm$0.01} &
  0.08\footnotesize{$\pm$0.04} &
  5.4\footnotesize{$\pm$1.8} &
  5.2\footnotesize{$\pm$2.9} &
  0.88\footnotesize{$\pm$0.03} &
  0.11\footnotesize{$\pm$0.04} &
  15.8\footnotesize{$\pm$5.8} &
  26.3\footnotesize{$\pm$8.2} &
  0.79\footnotesize{$\pm$0.15} &
  0.20\footnotesize{$\pm$0.05} &
  79.1\footnotesize{$\pm$18.8} &
  180.3\footnotesize{$\pm$54.0} \\
\multicolumn{2}{c|}{\textbf{CASPER(Ours)}} &
  \textbf{0.97\footnotesize{$\pm$0.03}} &
  0.03\footnotesize{$\pm$0.02} &
  \textbf{1.6\footnotesize{$\pm$1.1}} &
  \textbf{3.6\footnotesize{$\pm$2.5}} &
  \textbf{0.95\footnotesize{$\pm$0.05}} &
  0.13\footnotesize{$\pm$0.05} &
  \textbf{13.0\footnotesize{$\pm$5.4}} &
  \textbf{18.6\footnotesize{$\pm$10.2}} &
  \textbf{0.90\footnotesize{$\pm$0.05}} &
  0.15\footnotesize{$\pm$0.04} &
  \textbf{55.4\footnotesize{$\pm$15.4}} &
  \textbf{136.8\footnotesize{$\pm$48.4}} \\ \bottomrule
\end{tabular}%
}

\label{tab:linear-SF}
\end{table*}
\begin{table*}[b]
\centering
\caption{Nonlinear Setting, for SF graphs of 10, 20, 50 nodes.}
\resizebox{\textwidth}{!}{%
\begin{tabular}{cc|cccc|cccc|cccc}
\toprule
 &
   &
  \multicolumn{4}{c|}{10 nodes} &
  \multicolumn{4}{c|}{20 nodes} &
  \multicolumn{4}{c}{50 nodes} \\
\multicolumn{2}{c|}{SF2} & 
  TPR$\uparrow$ &
  FDR$\downarrow$ &
  SHD$\downarrow$ &
  SID$\downarrow$ &
  TPR$\uparrow$ &
  FDR$\downarrow$ &
  SHD$\downarrow$ &
  SID$\downarrow$ &
  TPR$\uparrow$ &
  FDR$\downarrow$ &
  SHD$\downarrow$ &
  SID$\downarrow$ \\ \midrule
\multicolumn{2}{c|}{Random} &
  0.11\footnotesize{$\pm$ 0.04} &
  0.89\footnotesize{$\pm$ 0.07} &
  32.2\footnotesize{$\pm$ 8.0} &
  32.1\footnotesize{$\pm$ 6.3} &
  0.11\footnotesize{$\pm$0.10} &
  0.89\footnotesize{$\pm$0.03} &
  43.2\footnotesize{$\pm$5.4} &
  96.8\footnotesize{$\pm$10.4} &
  0.06\footnotesize{$\pm$0.03} &
  0.89\footnotesize{$\pm$0.06} &
  354.2\footnotesize{$\pm$13.7} &
  1,297.5\footnotesize{$\pm$174.2} \\
\multicolumn{2}{c|}{NOTEARS-MLP} &
  0.84\footnotesize{$\pm$0.06} &
  0.25\footnotesize{$\pm$0.12} &
  6.7\footnotesize{$\pm$2.8} &
  9.1\footnotesize{$\pm$3.7} &
  0.42\footnotesize{$\pm$0.13} &
  0.23\footnotesize{$\pm$0.11} &
  25.3\footnotesize{$\pm$4.1} &
  47.9\footnotesize{$\pm$7.4} &
  0.23\footnotesize{$\pm$0.05} &
  \textbf{0.18\footnotesize{$\pm$0.02}} &
  74.8\footnotesize{$\pm$5.6} &
  266.5\footnotesize{$\pm$46.2} \\
\multicolumn{2}{c|}{DAG-GNN} &
  0.80\footnotesize{$\pm$0.12} &
  0.20\footnotesize{$\pm$0.03} &
  5.3\footnotesize{$\pm$2.1} &
  12.0\footnotesize{$\pm$5.6} &
  0.44\footnotesize{$\pm$0.12} &
  \textbf{0.15\footnotesize{$\pm$0.12}} &
  21.2\footnotesize{$\pm$4.5} &
  41.5\footnotesize{$\pm$8.1} &
  0.28\footnotesize{$\pm$0.02} &
  0.22\footnotesize{$\pm$0.03} &
  77.9\footnotesize{$\pm$4.8} &
  285.7\footnotesize{$\pm$36.6} \\
\multicolumn{2}{c|}{NoCurl} &
  0.78\footnotesize{$\pm$0.09} &
  0.15\footnotesize{$\pm$0.08} &
  6.0\footnotesize{$\pm$2.4} &
  11.5\footnotesize{$\pm$4.2} &
  0.32\footnotesize{$\pm$0.08} &
  0.17\footnotesize{$\pm$0.10} &
  23.4\footnotesize{$\pm$4.3} &
  45.4\footnotesize{$\pm$9.2} &
  0.20\footnotesize{$\pm$0.04} &
  0.19\footnotesize{$\pm$0.01} &
  80.5\footnotesize{$\pm$5.7} &
  293.6\footnotesize{$\pm$56.1} \\
\multicolumn{2}{c|}{Gran-DAG} &
  0.79\footnotesize{$\pm$0.05} &
  0.12\footnotesize{$\pm$0.04} &
  5.5\footnotesize{$\pm$3.3} &
  13.2\footnotesize{$\pm$3.8} &
  0.45\footnotesize{$\pm$0.11} &
  0.21\footnotesize{$\pm$0.09} &
  18.9\footnotesize{$\pm$6.2} &
  67.3\footnotesize{$\pm$6.5} &
  0.29\footnotesize{$\pm$0.07} &
  0.20\footnotesize{$\pm$0.06} &
  81.2\footnotesize{$\pm$4.3} &
  382.5\footnotesize{$\pm$49.6} \\
\multicolumn{2}{c|}{DARING} &
  0.82\footnotesize{$\pm$0.08} &
  0.08\footnotesize{$\pm$0.01} &
  5.4\footnotesize{$\pm$2.0} &
  8.1\footnotesize{$\pm$4.0} &
  0.37\footnotesize{$\pm$0.07} &
  0.19\footnotesize{$\pm$0.13} &
  27.8\footnotesize{$\pm$5.0} &
  42.1\footnotesize{$\pm$10.0} &
  0.33\footnotesize{$\pm$0.03} &
  0.21\footnotesize{$\pm$0.02} &
  70.8\footnotesize{$\pm$4.2} &
  302.0\footnotesize{$\pm$57.3} \\
\multicolumn{2}{c|}{\textbf{CASPER(Ours)}} &
  \textbf{0.89\footnotesize{$\pm$ 0.11}} &
  \textbf{0.06\footnotesize{$\pm$0.02}} &
  \textbf{4.8\footnotesize{$\pm$1.6}} &
  \textbf{6.5\footnotesize{$\pm$3.3}} &
  \textbf{0.51\footnotesize{$\pm$0.13}} &
  0.16\footnotesize{$\pm$0.13} &
  \textbf{13.7\footnotesize{$\pm$6.3}} &
  \textbf{30.2\footnotesize{$\pm$9.4}} &
  \textbf{0.41\footnotesize{$\pm$0.05}} &
  0.24\footnotesize{$\pm$0.06} &
  \textbf{64.1\footnotesize{$\pm$6.2}} &
  \textbf{201.3\footnotesize{$\pm$38.5}} \\ \midrule \midrule
\multicolumn{2}{c|}{SF4} &
  TPR$\uparrow$ &
  FDR$\downarrow$ &
  SHD$\downarrow$ &
  SID$\downarrow$ &
  TPR$\uparrow$ &
  FDR$\downarrow$ &
  SHD$\downarrow$ &
  SID$\downarrow$ &
  TPR$\uparrow$ &
  FDR$\downarrow$ &
  SHD$\downarrow$ &
  SID$\downarrow$ \\ \midrule 
\multicolumn{2}{c|}{Random} &
  0.10\footnotesize{$\pm$0.02} &
  0.93\footnotesize{$\pm$0.13} &
  68.2\footnotesize{$\pm$9.3} &
  76.1\footnotesize{$\pm$6.7} &
  0.09\footnotesize{$\pm$0.04} &
  0.87\footnotesize{$\pm$0.05} &
  159.2\footnotesize{$\pm$13.8} &
  185.6\footnotesize{$\pm$26.2} &
  0.03\footnotesize{$\pm$0.02} &
  0.91\footnotesize{$\pm$0.04} &
  1,098.5\footnotesize{$\pm$79.5} &
  1,989.9\footnotesize{$\pm$184.2} \\
\multicolumn{2}{c|}{NOTEARS-MLP} &
  0.74\footnotesize{$\pm$0.09} &
  0.24\footnotesize{$\pm$0.05} &
  13.2\footnotesize{$\pm$3.9} &
  19.6\footnotesize{$\pm$4.9} &
  0.25\footnotesize{$\pm$0.04} &
  0.24\footnotesize{$\pm$0.09} &
  58.6\footnotesize{$\pm$3.5} &
  115.4\footnotesize{$\pm$21.0} &
  0.29\footnotesize{$\pm$0.02} &
  0.13\footnotesize{$\pm$0.01} &
  163.8\footnotesize{$\pm$10.2} &
  453.8\footnotesize{$\pm$115.6} \\
\multicolumn{2}{c|}{DAG-GNN} &
  0.80\footnotesize{$\pm$0.14} &
  \textbf{0.12\footnotesize{$\pm$0.03}} &
  12.1\footnotesize{$\pm$3.4} &
  9.4\footnotesize{$\pm$7.4} &
  0.34\footnotesize{$\pm$0.05} &
  0.17\footnotesize{$\pm$0.05} &
  48.1\footnotesize{$\pm$4.1} &
  128.6\footnotesize{$\pm$21.8} &
  0.23\footnotesize{$\pm$0.03} &
  0.12\footnotesize{$\pm$0.02} &
  143.2\footnotesize{$\pm$8.4} &
  336.4\footnotesize{$\pm$91.5} \\
\multicolumn{2}{c|}{NoCurl} &
  0.79\footnotesize{$\pm$0.03} &
  0.24\footnotesize{$\pm$0.04} &
  13.1\footnotesize{$\pm$2.0} &
  13.5\footnotesize{$\pm$6.5} &
  0.36\footnotesize{$\pm$0.06} &
  0.12\footnotesize{$\pm$0.03} &
  49.0\footnotesize{$\pm$3.6} &
  135.0\footnotesize{$\pm$26.7} &
  0.26\footnotesize{$\pm$0.03} &
  \textbf{0.05\footnotesize{$\pm$0.04}} &
  180.5\footnotesize{$\pm$7.3} &
  549.6\footnotesize{$\pm$57.8} \\
\multicolumn{2}{c|}{Gran-DAG} &
  0.82\footnotesize{$\pm$0.15} &
  0.26\footnotesize{$\pm$0.05} &
  8.9\footnotesize{$\pm$1.7} &
  12.0\footnotesize{$\pm$6.2} &
  0.24\footnotesize{$\pm$0.09} &
  \textbf{0.09\footnotesize{$\pm$0.02}} &
  54.2\footnotesize{$\pm$8.2} &
  133.2\footnotesize{$\pm$25.2} &
  0.21\footnotesize{$\pm$0.08} &
  0.14\footnotesize{$\pm$0.06} &
  171.4\footnotesize{$\pm$6.5} &
  562.7\footnotesize{$\pm$55.0} \\
\multicolumn{2}{c|}{DARING} &
  0.80\footnotesize{$\pm$0.08} &
  0.22\footnotesize{$\pm$0.06} &
  10.5\footnotesize{$\pm$2.3} &
  12.8\footnotesize{$\pm$5.1} &
  0.33\footnotesize{$\pm$0.07} &
  0.15\footnotesize{$\pm$0.08} &
  55.3\footnotesize{$\pm$4.1} &
  115.1\footnotesize{$\pm$24.0} &
  0.30\footnotesize{$\pm$0.05} &
  0.18\footnotesize{$\pm$0.02} &
  153.6\footnotesize{$\pm$8.5} &
  384.6\footnotesize{$\pm$42.7} \\
\multicolumn{2}{c|}{\textbf{CASPER(Ours)}} &
  \textbf{0.88\footnotesize{$\pm$0.06}} &
  0.15\footnotesize{$\pm$0.02} &
  \textbf{7.9\footnotesize{$\pm$2.8}} &
  \textbf{6.5\footnotesize{$\pm$3.7}} &
  \textbf{0.44\footnotesize{$\pm$0.05}} &
  0.12\footnotesize{$\pm$0.03} &
  \textbf{40.1\footnotesize{$\pm$5.3}} &
  \textbf{97.8\footnotesize{$\pm$23.3}} &
  \textbf{0.41\footnotesize{$\pm$0.02}} &
  0.21\footnotesize{$\pm$0.07} &
  \textbf{123.5\footnotesize{$\pm$7.7}} &
  \textbf{298.2\footnotesize{$\pm$84.5}} \\ \bottomrule
\end{tabular}%
}

\label{tab:nonlinear-SF}
\end{table*}
\begin{table*}
\caption{Empirical results for running time (sec comparison) on ER2 graph of 10 nodes (Linear setting).}
\centering
\resizebox{0.75\textwidth}{!}{%
\begin{tabular}{cccccccc}
\hline
Method    & NOTEARS & NOTEARS-MLP & GraN-DAG & NoCurl & DAG-GNN & DARING & \textbf{CASPER} \\ \hline
Time(sec) & 1.32    & 2.12        & 23.30    & 1.01   & 10.30   & 3.21   & \textbf{2.32}   \\ \hline
\end{tabular}%
}

\label{tab:time1}
\end{table*}

\begin{table*}
\caption{Empirical results for running time (sec comparison) on ER2 graph of 10 nodes (Nonlinear setting).}
\centering
\resizebox{0.75\textwidth}{!}{%
\begin{tabular}{cccccccc}
\hline
Method    & NOTEARS & NOTEARS-MLP & GraN-DAG & NoCurl & DAG-GNN & DARING & \textbf{CASPER} \\ \hline
Time(sec) & 3.32    & 4.15        & 33.10    & 2.31   & 16.32   & 6.12   & \textbf{4.89}   \\ \hline
\end{tabular}%
}

\label{tab:time2}
\end{table*}

\begin{table*}
\caption{ Empirical results for running time (sec) comparison
on Sachs~\cite{sachs2005causal} dataset.}
\centering
\resizebox{0.75\textwidth}{!}{%
\begin{tabular}{cccccccc}
\hline
Method    & NOTEARS & NOTEARS-MLP & GraN-DAG & NoCurl & DAG-GNN & DARING & \textbf{CASPER} \\ \hline
Time(sec) & 3.97    & 4.02        & 35.30    & 3.25   & 15.30   & 6.58   & \textbf{5.32}   \\ \hline
\end{tabular}%
}

\label{tab:time3}
\end{table*}
\end{document}